\documentclass[nohyperref]{article}

\usepackage{microtype}
\usepackage{graphicx}
\usepackage{subfigure}
\usepackage{booktabs} % for professional tables

\usepackage[table,xcdraw,svgnames]{xcolor}
\usepackage[colorlinks]{hyperref}
\AtBeginDocument{%
  \hypersetup{
    citecolor=[RGB]{50,100,170},
    linkcolor=[RGB]{50,100,170},
    urlcolor=[RGB]{255,102,178}}}
\usepackage[accepted]{icml2023}

\usepackage{amsmath}
\usepackage{amssymb}
\usepackage{mathtools}
\usepackage{amsthm}

\usepackage{thmtools}
\usepackage{thm-restate}

\usepackage{definitions}

\usepackage[capitalize,noabbrev]{cleveref}

\theoremstyle{plain}
\usepackage[textsize=tiny]{todonotes}

\newcommand{\pistar}{\pi^\star}
\newcommand{\Mstar}{M^\star}
\newcommand{\pihat}{\hat{\pi}}
\newcommand{\Mhat}{\hat{M}}
\newcommand{\expec}{\mathop{{}\mathbb{E}}}
\newcommand{\gexpec}{\mathop{{}\gamma\mathbb{E}}}

\newcommand{\ouralg}{\textsc{LAMPS}}
\newcommand{\ouralgmm}{\textsc{LAMPS-MM}}
\newcommand{\sysid}{\textsc{SysID}}
\newcommand{\mbpo}{\textsc{MBPO}}
\newcommand{\mbposysid}{\textsc{MBPO-SysID}}
\newcommand{\mbposysidcompute}{\textsc{MBPO-SysID (2x)}}
\newcommand{\sprime}{s^{\prime}}
\newcommand{\sdprime}{s^{\prime\prime}}
\newcommand{\pdam}{\textsc{PDAM}}

\definecolor{green2}{rgb}{0.1,0.62,0.46}
\definecolor{orange2}{rgb}{0.85,0.37,0.0}

\newcommand{\pistarc}{ {\pistar}}
\newcommand{\pihatc}{ {\pihat}}

\icmltitlerunning{~\hfill The Virtues of Laziness in Model-based RL\hfill \thepage}

\begin{document}

\twocolumn[
% \icmltitle{A New Approach to Model-Based Reinforcement Learning}
% \icmltitle{Overcoming Objective Mismatch and Computational
% Inefficiency in Model-Based RL}
\icmltitle{The Virtues of Laziness in Model-based RL: A
  Unified Objective and Algorithms}

% 1. Efficient model-based RL does not require learning accurate models
% 2. Improvements in model-based RL: Statistical and Computational gains using simulated performance difference
% 3. Faster model-based RL: Improvements in computational and statistical efficiency




\begin{icmlauthorlist}
\icmlauthor{Anirudh Vemula}{aurora}
\icmlauthor{Yuda Song}{cmu}
\icmlauthor{Aarti Singh}{cmu}
\icmlauthor{J. Andrew Bagnell}{aurora,cmu}
\icmlauthor{Sanjiban Choudhury}{cornell}
\end{icmlauthorlist}

\icmlaffiliation{aurora}{Aurora Innovation, Pittsburgh, PA USA}
\icmlaffiliation{cmu}{School of Computer Science, Carnegie Mellon
  University, Pittsburgh, PA USA}
\icmlaffiliation{cornell}{Computer Science, Cornell University, Ithaca
NY USA}

\icmlcorrespondingauthor{Anirudh Vemula}{vvanirudh@gmail.com}
%\icmlcorrespondingauthor{Firstname2 Lastname2}{first2.last2@www.uk}

\icmlkeywords{reinforcement learning, model-based reinforcement learning}

\vskip 0.3in
]



\printAffiliationsAndNotice{} % otherwise use the standard text.

\begin{abstract}
We propose a novel approach to addressing two fundamental challenges
in Model-based Reinforcement Learning (MBRL): the computational
expense of repeatedly finding a good policy in the learned model, and the
objective mismatch between model fitting and policy computation.  Our
``lazy'' method leverages a novel unified objective,
\textit{Performance Difference via Advantage in Model}, to capture the
performance difference between the learned policy and expert policy
under the true dynamics. This objective demonstrates that optimizing the
expected policy advantage in the learned model under an exploration
distribution is sufficient for policy computation, resulting in a
significant boost in computational efficiency compared to traditional
planning methods. Additionally, the unified objective uses a value
moment matching term for model fitting, which is aligned with the
model's usage during policy computation. We present two no-regret
algorithms to optimize the proposed objective, and demonstrate their
statistical and computational gains compared to existing MBRL methods
through simulated benchmarks.

%%% Local Variables:
%%% mode: latex
%%% TeX-master: "../example_paper"
%%% End:

\end{abstract}

\section{Introduction}
%% OUTLINE
% Why Model-based RL? Describe the general framework
Model-based Reinforcement Learning (MBRL) methods show great promise
for real world applicability as they often require remarkably fewer
number of real world interactions compared to model-free
counterparts~\cite{schrittwieser2020mastering,hafner2023mastering}. The
key idea is that, in contrast to model-free RL that computes a policy
directly from real world data, we can perform the following iterative
procedure~\cite{sutton2018reinforcement}: we fit a model that
accurately predicts the dynamics on the data collected so far using
the learned policy.  Subsequently, we compute a policy through optimal
planning in the learned model, and use it to collect more data in the
real world. This procedure is repeated until a satisfactory policy is
learned. Theoretical studies, such as
\citet{ross2012agnostic}, have shown that this procedure can find a
near-optimal policy in a statistically efficient manner under certain
conditions, such as access to a good exploration distribution and a
rich enough model class, and this has been validated by its good
performance in practice.

%\jab{We might want to rephrase this because the theory drove the exploration distribution use here.}

% Highlight the two major issues of MBRL
However, there are two major challenges with the above procedure. The
policy computation step in each iteration relies on solving the computationally
expensive problem of finding the best policy in the learned
model. This can require a number of interactions in the model that is
exponential in the task
horizon~\cite{kearns1999approximate}. Furthermore, past
literature including~\citet{ross2012agnostic,jiang2018imperfect,vemula2020planning} among others have shown that
optimal planning in learned models can result in policies that exploit
inaccuracies in the learned model hindering fast learning and
statistical efficiency.
% \jab{I would argue this is the central point of the SysID paper as well
% as https://youtu.be/xfyK03MEZ9Q?t=17095 . But I learned it from Chris Atkeson and tried to turn the
% rough idea into a formal one: that we can think about model based learning as needing
% to be robust to adversarial behavior. There's probably older references; suggest asking Chris
% but might be worth including the two above as they are the first I know that formalize
% the "my planner tries to exploit my model" problem.}

The second challenge pertains to the objective mismatch between model fitting and
policy computation that is extensively studied in recent
literature~\cite{farahmand2017value,lambert2020mismatch}. The model
fitting objective of minimizing prediction error is not necessarily related
to the objective of maximizing the performance of the policy, derived
from the model, in the real world. This results in a mismatch of
objectives used to fit the model and how the model is used when
computing the policy through planning. This is exacerbated in cases
where the model class is not realizable, i.e. no model in the model class
can perfectly explain true dynamics, which is often the case in real
world tasks~\cite{joseph2013rl}.

% Summarize our approach and list major contributions
% In this work, we revisit the objective proposed
% in~\cite{ross2012agnostic} and derive a new decomposition using our
% proposed \textit{Simulation Performance
%   Difference Lemma.} This results in a unified objective that informs
% two major changes to the existing MBRL procedure. Instead of computing
% the optimal policy in
% learned model at each iteration, we optimize the expected
% policy advantage in the model under the exploration distribution which
% only requires a number of interactions in the model that is polynomial
% in the task horizon. For model fitting, our new objective measures how
% well the predicted and observed next states, in the data collected so
% far, match in terms of their value function in the learned model. This
% ensures that we update the model to be accurate in states that are
% critical in policy computation, and allow the model to be inaccurate
% in states that are irrelevant. Thus, our proposed
% unified objective encourages ``laziness'' in both steps of the MBRL
% procedure, and solves both the computational expense and objective
% mismatch issues.
In this work, we propose a new decomposition of the performance difference between the learned policy and expert policy under true dynamics, which we coin as \textit{Performance
  Difference via Advantage in Model}. 
%\jab{Confused how it's a new decomposition of that old objective.}  
This leads to a unified objective that informs two major changes to the
existing MBRL procedure. Instead
of computing the optimal policy in the learned model at each
iteration, we optimize the expected policy advantage in the model
under an exploration distribution which only requires a number of
interactions in the model that is polynomial in the task horizon. For
model fitting, our new objective measures the similarity of predicted
and observed next states in terms of their value function in the
learned model. This ensures that the model is updated to be accurate
in states that are critical for policy computation, and allows for
inaccuracy in states that are irrelevant. Therefore, our proposed
unified objective encourages ``laziness'' in both steps of the MBRL
procedure, solving both the \textbf{computational expense} and
\textbf{objective mismatch} challenges.

Our contributions in this paper are as follows:
\begin{itemize}
\item A unified objective for MBRL that is both computationally more
  efficient in policy computation and resolves the objective mismatch
  in model fitting.
\item Two algorithms that leverage the laziness in the proposed
  objective to achieve tighter performance bounds than those of \citet{ross2012agnostic}.
% \jab{This confuses things a bit. We're not creating two no regret algorithms. We're
% providing two meta-approaches that use online interaction (or no regret learner) to achieve
% tighter performance bounds w.r.t. an optimal policy.}
\item An empirical demonstration through simulated benchmarks that our
  proposed algorithms result in both statistical and computational
  gains compared to existing MBRL methods.
\end{itemize}

% Outline of paper, if needed

%%% Local Variables:
%%% mode: latex
%%% TeX-master: "../example_paper"
%%% End:

\section{Related Work}\label{sec:related}
% This work is most closely related to~\citet{ross2012agnostic}
% which established a theoretical justification for MBRL framework
% using a reduction to no-regret online learning. In this
% section, we survey other relevant previous works.

\textbf{Model-based RL}
% Model-based RL and model-based optimal control involve a long line
% of research~\citep{ljung1998system,morari1999model,sutton1991dyna}. On
% the practitioner side, in addition to the previous achievements on the
% low-dimensional state space
% \citep{levine2014learning,chua2018deep,schrittwieser2020mastering},
% there are also recent advances on the high-dimensional (image) state
% space \citep{hafner2020mastering, wu2022daydreamer}. Model-based
% methods' performance guarantee and the better sample complexity
% complexity than model-free counterparts are also been studied and
% proved in the theoretical side of works
% \citep{abbasi2011regret,ross2012agnostic,tu2019gap,sun2019model}. However,
% although the model fitting part is different among the previous works,
% one always needs to compute the optimal policy from the learned model,
% from the easiest ones using value iteration \citep{azar2013minimax} to
% the strongest ones requiring a black-box policy optimization oracle
% \citep{kakade2020information,song2021pc}. In this work, we show that
% computing the optimal policy in the learned model is not necessary and
% propose a computation efficient alternative without sacrificing
% performance guarantee.
Model-based Reinforcement Learning and Optimal Control have been
extensively researched in the literature, with a long line of
works~\cite{ljung1998system, morari1999model, sutton1991dyna}. Recent
work has made significant achievements in tasks with both
low-dimensional state spaces~\cite{levine2014learning, chua2018deep,
  schrittwieser2020mastering} and high-dimensional state
spaces~\cite{hafner2020mastering, wu2022daydreamer}. Theoretical
studies have also been conducted to analyze the performance guarantees
and sample complexity of model-based methods~\cite{abbasi2011regret,
  ross2012agnostic, tu2019gap, sun2019model}. However, there is a
common requirement among previous works to compute the optimal policy
from the learned model at each iteration, using methods that range
from value iteration~\cite{azar2013minimax} to black-box policy
optimization~\cite{kakade2020information, song2021pc}. In this work,
we show that computing the optimal policy in the learned model is not
necessary and propose a computationally efficient alternative that
does not compromise performance guarantees.

\textbf{RL with exploration distribution}
% In this work and \citet{ross2012agnostic}, we assume that it is easy
% to have access to an exploration distribution. Recently works under
% the setting of Hybrid RL
% \citep{rajeswaran2017learning,vecerik2017leveraging,xie2021policy,song2022hybrid}
% study the setting where both an offline dataset and online environment
% interaction access are granted. Although the intuitions are similar,
% we highlight a few core differences: 1) some of the works require the
% offline dataset is collected from the expert policy. 2) All of the
% hybrid RL methods are model-free, and the analysis in the model-based
% setting is still lacking, except \citet{ross2012agnostic}. 3) The
% practical performance (and theoretical guarantee) of the previous
% works are upper bounded by the quality of the offline dataset, where
% in our experiment we observe that our method indeed outperforms the
% explore dataset under some tasks.
In this work, as well as in previous works such
as \citet{kakade2002approximate,bagnell2003psdp,ross2014aggrevate}, we
assume access to an
exploration distribution that allows us to exploit any prior
knowledge of the task to learn good policies
quickly. We also leverage a similar model-free policy search algorithm
in this work within the MBRL framework.
%\todo[inline]{@Yuda: Cut down the Hybrid RL works below to 3-4 lines}
Recent works in the field of Hybrid Reinforcement
Learning~\cite{rajeswaran2017learning,vecerik2017leveraging,nair2018overcoming,hester2018deep,xie2021policy,song2022hybrid}
consider a related setting where both an offline dataset and online
access to interact with environment are available. However, most of these
works are in the model-free setting and require that the offline
dataset is collected from an expert policy while our model-based
setting only requires an exploration distribution that covers
the expert distribution.
%Also, all of the Hybrid RL
%methods are model-free, and the analysis in the model-based setting is
%still lacking except for~\cite{ross2012agnostic}.

\textbf{Objective Mismatch in MBRL}
% Recently a line of works
% \citep{farahmand2017value,lambert2020mismatch,
%   ziebart2010modeling,eysenbach2021mismatched} identified the
% objective mismatch issue in model-based RL, that is the mismatch
% between the training objective (finding the MLE model) and the final
% objective (finding the optimal policy). One way to reduce such
% objective mismatch is to incorporate certain value functions into the
% model training objective. For example,
% \citet{farahmand2017value,grimm2020value,grimm2021proper} proposed to
% find models that can correctly predict the expected successor values
% over a set of value functions and policies. \citet{modhe2021model}
% studied model advantage instead of values. \citet{grimm2020value}
% proposed to use the value-incorporated model loss to construct the
% confidence class, instead of the MLE objective. In our work, we
% proposed a cleaner and more intuitive objective that only requires the
% model to match the value moments of the value functions in the past
% iterations.
Recent
works~\cite{farahmand2017value,lambert2020mismatch,voloshin2021minimax,eysenbach2021mismatched}
 identified an objective mismatch issue in MBRL, where there is a
mismatch between the
training objective (finding the maximum likelihood estimate
model) and the true objective (finding the optimal policy in real
world). To
address this issue, several works have proposed to incorporate
value-aware objectives during model
fitting. \citet{farahmand2017value,grimm2020value,voloshin2021minimax}
proposed to
find models that can correctly predict the expected successor values
over a pre-defined set of value functions and policies. \citet{modhe2021model}
used model advantage under the learned policy as the objective for model
fitting and use planning to compute the policy. \citet{ayoub2020model} present a similar approach where model fitting uses a value targeted regression objective and leverage optimism to only choose models that are consistent with the data collected so far. However, their approach assumes realizability in the model class, and requires solving an optimistic planning
problem with the constructed set of models.
% ~\cite{grimm2020value}
% proposed to use the value-incorporated model loss to construct a
% confidence class, instead of the MLE objective.
Instead, we propose a unified objective for both policy and model
learning from first principles that is both value-aware and feasible
to optimize using no-regret algorithms.

%%% Local Variables:
%%% mode: latex
%%% TeX-master: "../example_paper"
%%% End:

\section{Preliminaries}
\label{sec:setup}
We assume the real world behaves according to an infinite horizon
discounted Markov Decision Process (MDP)
$(\Scal, \Acal, M^\star, \omega, c, \gamma)$, where $\Scal$ is the
state space, $\Acal$ is the action space,
$\Mstar: \Scal \times \Acal \to \Delta(\Scal)$ is the transition
dynamics, $c: \Scal \times \Acal \to [0,1]$ is the cost function,
$\gamma$ is the discount factor, and $\omega \in \Delta(\Scal)$ is the
initial state distribution with $\Delta(\Scal)$ defining the set of
probability distributions on set $\Scal$. The true dynamics $\Mstar$
is unknown but we can collect data in real world. We assume cost
function $c$ is known, but our results can be extended to the case
where $c$ is unknown.

For any policy $\pi: \Scal \to \Delta(\Acal)$, we denote
$D^h_{\omega, \pi}$ as state-action distribution at time $h$ if we
started from an initial state sampled from $\omega$ and executed $\pi$
until time $h-1$ in $\Mstar$. This can be generalized using
$D_{\omega, \pi} = (1 - \gamma)\sum_{h=1}^{\infty}
\gamma^{h-1}D_{\omega, \pi}^h$ which is the state-action distribution
over the infinite horizon.
% if we follow $\pi$ starting from the initial
% distribution $\omega$.
In a similar fashion, we will use the notation
$d_{\omega, \pi}$ to denote the infinite horizon state distribution. %if we follow $\pi$ starting
%from the initial distribution $\omega$ in $\Mstar$.
 We denote the value function of policy $\pi$ under any transition
 function $M$ as $V^\pi_M(s)$, the state-action value function as
 $Q^\pi_M(s, a) = c(s, a) + \expec_{s' \sim M(s, a)}V^\pi_M(s')$, and
 the performance is defined as
 $J_M^\omega(\pi) = \mathbb{E}_{s \sim \omega}[V^\pi_M(s)]$. The goal
 is to find a policy
 $\pistar = \argmin_{\pi \in \Pi}J_{\Mstar}^\omega(\pi)$.

Similar to \citet{ross2012agnostic}, our approach assumes access to a
state-action exploration distribution
$\nu$ to sample from and allows us to guarantee small regret
against any policy with a state-action distribution close to $\nu$. If
$\nu$ is close to $D_{\omega, \pistar}$,
then our approach guarantees near-optimal performance. Good
exploration distributions can often be obtained in practice either
from expert demonstrations, domain knowledge, or from a desired
trajectory that we want the system to follow.

% We assume that we either have the expert distribution $D^{\omega,
%   \pi^\ast}$, or an explore distribution that covers $D^{\omega,
%   \pi^\ast}$ in the point-wise sense. For simplicity, our analysis
% directly used $D^{\omega, \pi^\ast}$, but we will show that this is
% without loss of generality in \pref{sec:alg_mle}.

%Finally,through out the paper we define distances as $\|\cdot\|$, for
% example, $\|\cdot\|_1$ as $\ell_1$ norm and $\|\cdot,\cdot\|_{KL}$ as
% the KL-distance.

%%% Local Variables:
%%% mode: latex
%%% TeX-master: "../example_paper"
%%% End:

\subsection{MBRL Framework}\label{sec:framework}
% In this section, we describe the general framework of MBRL and
% highlight the two major issues with existing MBRL methods, which will
% motivate our approach introduced in the following section. All formal
% proofs in this section will be deferred to the appendix.

{\begin{algorithm}[t]
\caption{Meta algorithm for MBRL}
\begin{algorithmic}[1]
\REQUIRE Number of iterations $T$, model class $\Mcal$, Policy class
$\Pi$, exploration distribution $\nu$
\STATE Initialize model $M_1 \in \Mcal$
\STATE Compute policy $\pihat_1$ using \textbf{ComputePolicy}
\FOR{$t = 1, \dots, T$}
%\textcolor{blue}{\# Collect online dataset}  \\
\STATE Collect data in $\Mstar$ by rolling out $\pihat_t$ or sampling
from $\nu$ (with equal prob.) and add to dataset $\Dcal_t$
\STATE Fit model $\hat M_{t+1}$ to $\Dcal_t$ using \textbf{FitModel}
\STATE Compute policy $\pihat_{t+1}$ using \textbf{ComputePolicy}
\ENDFOR
\STATE \textbf{Return} Sequence of policies $\{\pihat_t\}_{t=1}^{T+1}$
\end{algorithmic}\label{alg:meta}
\end{algorithm}}

The MBRL framework of \citet{ross2012agnostic} is described as a meta
algorithm in \pref{alg:meta}. Starting with an exploration
distribution, at each iteration we collect data using both the
learned policy and the exploration
distribution, fit a model to the data collected so far, and compute a
policy using the newly learned model. Note that the model fitting and
policy computation procedures in \pref{alg:meta} are abstracted for
now.
% Most existing methods fit the model by
% maximizing the likelihood of observed data under the learned
% model~\cite{ljung1998system,deisenroth2011pilco}. For computing the
% policy, existing methods use a variety of optimal planning methods
% depending on the model class used.

To understand why \pref{alg:meta} would result in a policy that has
good performance in the real world $\Mstar$, let us revisit the
objective presented in \citet{ross2012agnostic}. This objective is a
result of applying an essential tool in MBRL
analysis,~\cite{kearns2002near} the \textit{Simulation Lemma} (see
\pref{lem:simulation},) twice to compute the performance difference of
any two policies $\hat{\pi}, \pistar$ in the real world $\Mstar$,
which is the quantity of interest we would like to optimize. In other
words, we would like to find a policy $\pihat$ whose performance in
$\Mstar$ is close to that of the expert $\pistar$ in
$\Mstar$.

\begin{lemma}[Performance Difference via Planning in Model]\label{lem:double}
    For any start state distribution $\omega$, policies $\hat{\pi}$,
    $\pistar$, and transition functions $\Mhat, \Mstar$ we have,
    {\small\begin{align}
        (1-&\gamma)[J_{\Mstar}^\omega(\pihatc) -
          J_{\Mstar}^\omega(\pistarc)] =  \nonumber\\
             &\underbrace{(1-\gamma)\expec_{s \sim \omega}[V^{\pihatc}_{\Mhat}(s) -
          V^{\pistarc}_{\Mhat}(s)]\label{term:dlm_1}}_{\text{Performance
               difference in the Model}}\\
        &+\underbrace{\gexpec_{\substack{(s, a) \sim D_{\omega, \pihatc} \\
             \sprime \sim \Mstar(s, a)}}[V_{\Mhat}^{\pihatc}(\sprime) - \expec_{
      \sdprime \sim \Mhat(s,
          a)}[V_{\Mhat}^{\pihatc}(\sdprime)]] \label{term:dlm_2}}_{\text{Value
             difference on states visited by learned policy}}\\
        &+ \underbrace{\gexpec_{\substack{(s, a) \sim D_{\omega, \pistarc} \\
             \sprime \sim \Mstar(s, a)}}[\expec_{
          \sdprime \sim \Mhat(s,
          a)}[V_{\Mhat}^{\pistarc}(\sdprime)] -
             V_{\Mhat}^{\pistarc}(\sprime)]}_{\text{ (Expert) Value
             difference on states visited by expert}} \label{term:dlm_3}
    \end{align}}%
\end{lemma}
The above lemma tells us that the performance difference can be
decomposed into a sum of
three terms: term~\eqref{term:dlm_1} is the performance difference
between the two policies in the learned model $\Mhat$, and
terms~\eqref{term:dlm_2} and~\eqref{term:dlm_3} capture the difference in
values of the next states induced by the learned model and real world along
trajectories sampled from $\pihat$ and $\pistar$ in the real world
$\Mstar$ respectively. Term~\eqref{term:dlm_1} can be made small by
ensuring that the learned policy $\pihat$ achieves low costs in the
learned model $\Mhat$ by, for example, running optimal planning
in $\Mhat$ such that
\begin{equation}
  \label{eq:oc}
  \expec_{s \sim \omega} [V^{\hat \pi}_{\hat M}(s)] - \min_{\pi \in
    \Pi}\expec_{s \sim \omega}[V^{\pi}_{\Mhat}(s)] \leq
  \epsilon_{oc}
\end{equation}
Terms~\eqref{term:dlm_2} and~\eqref{term:dlm_3} can be made small
if the model has a low prediction error. This is formalized in the
corollary below by applying H\"older's inequality to these terms:

\begin{corollary}\label{cor:tv}
    For any start state distribution $\omega$,
    transition functions $\Mhat, \Mstar$, and policies $\pistar, \pihat$
    such that $\pihat$ satisfies~\eqref{eq:oc}, we have,
    {\begin{align*}
        (1-\gamma)&[J_{\Mstar}^\omega(\pihat) - J_{\Mstar}^\omega(\pistar)] \leq
                                     \epsilon_{oc}\\
                                   &+\gamma\hat{V}_{\max}\expec_{(s,
                                                          a) \sim
                                                          D_{\omega,
                                                          \pihat}}\left\|\Mhat(s,
                                                          a)-\Mstar(s,
                                                          a)\right\|_1
      \\&+\gamma V_{\max}\expec_{(s, a) \sim
      D_{\omega, \pistar}}\left\|\Mhat(s, a)-\Mstar(s, a)\right\|_1,
    \end{align*}}%
where $\hat V_{\max}=\|V_{\hat M}^{\hat \pi}\|_{\infty},
V_{\max} = \|V^{\pi^\ast}_{\hat M}\|_{\infty}$.
\end{corollary}

\begin{algorithm}[t]
\caption{MLE \textbf{FitModel}($\Dcal_t, \{\ell_i\}_{i=1}^{t-1}$)}
\begin{algorithmic}[1]
\REQUIRE Data $\Dcal_t$, model class $\Mcal$, previous losses $\{\ell_i\}_{i=1}^{t-1}$
\STATE Define loss $\ell_t(M)=\expec_{(s, a, s') \sim \Dcal_t}\log
M(s'|s, a)$
\STATE Compute model $\Mhat_{t+1}$ using an online no-regret
algorithm, such as Follow-the-Leader (FTL)~\cite{hazan2019oco},
\begin{align*}
    \hat M_{t+1} \leftarrow \argmin_{M \in \Mcal} \sum_{\tau=1}^t  \ell_\tau(M).
\end{align*}
\STATE \textbf{Return} $\hat M_{t+1}$
\end{algorithmic}\label{alg:model_mle}
\end{algorithm}

Most MBRL methods including \citet{ross2012agnostic} use maximum
likelihood estimation (MLE)~\pref{alg:model_mle} to bound the total
variation loss terms in Corollary~\ref{cor:tv}\footnote{We can further bound
  the total variation terms using KL divergence through Pinsker's
  inequality. Then maximizing likelihood of observed data under
  learned model would
  minimize the KL divergence.} and use optimal planning approaches to satisfy equation~\eqref{eq:oc}. Combining \pref{alg:meta} with
\pref{alg:model_mle} and equation~\eqref{eq:oc} gives us a template for
understanding existing MBRL methods.

\begin{figure}[t]
  \centering
  \includegraphics[width=0.5\linewidth]{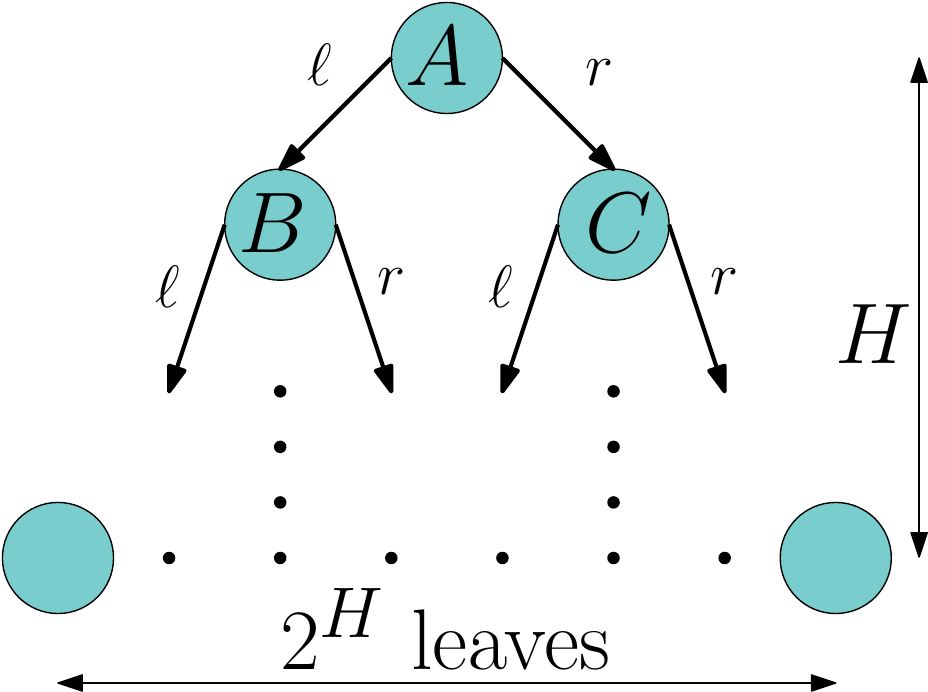}
  \caption{MDP with two actions $\ell$ and
    $r$, and the true dynamics $\Mstar$ are shown in the figure. The
    cost $c(s,a) = \epsilon<<1$ at any $s \neq B$ and
    $c(B, a) = 1$, for any action $a$. Thus, the action taken at $A$ is
    critical. Model class $\Mcal$ contains only two models:
    $M^{\mathsf{good}}$ which captures dynamics at $A$ correctly but makes
    mistakes everywhere else, while $M^{\mathsf{bad}}$ makes mistakes only at
    $A$ but captures true dynamics everywhere else.}
\label{fig:toy}
%\vspace{-0.5cm}
\end{figure}

\subsection{Challenges in MBRL}
\label{sec:challenges}
%\SAN{This is a bit too verbose. Can we jump straight in and define C1 and C2?}

% From Section~\ref{sec:framework}, we can make two important
% observations: first, we require that the policy computation step in
% MBRL framework results in a policy $\pihat$ that
% satisfies~\eqref{eq:oc} with a small error $\epsilon_{oc}$. Second,
% bounding terms~\ref{term:dlm_2} and~\ref{term:dlm_3} using model
% prediction error as in Corollary~\ref{cor:tv} results in an objective
% that is a weak upper bound of the original objective
% in~\pref{lem:double}. These two observations directly manifest as the
% two challenges in modern MBRL as explained below. We will use
% Figure~\ref{fig:toy} to motivate the challenges.

We make two important observations from the previous section which
directly manifest as the two fundamental challenges in MBRL. We will
use~\pref{fig:toy} to motivate these challenges. First,
ensuring that the learned policy $\pihat_t$ satisfies~\eqref{eq:oc} in
\pref{alg:meta} requires performing optimal planning in the learned
model $\Mhat_t$ at \textit{every iteration}. Note that optimal
planning can require a number of interactions in $\Mhat_t$ that is
exponential in the effective task horizon
$\frac{1}{1-\gamma}$~\cite{kearns1999approximate}. For example in
Figure~\ref{fig:toy}, solving for optimal policy requires $\Ocal(2^H)$
operations. We term this as \textbf{C1:~computational expense}
challenge in MBRL.

Second, \pref{lem:double} indicates that optimizing
terms~\eqref{term:dlm_2} and~\eqref{term:dlm_3} (along
with~\eqref{eq:oc}) is guaranteed to optimize the performance
difference. However, we cannot directly optimize
term~\eqref{term:dlm_3} as it requires access to the value function of
the expert in the model $V^{\pistar}_{\Mhat}$ which is unknown. To
avoid this, MBRL methods bound these terms using model prediction
error in Corollary~\ref{cor:tv} as an upper bound that is easy to
optimize but very loose, especially due to the unknown scaling term
$\|V_{\Mhat}^{\pistar}\|_\infty$ which can be as large as
$\frac{1}{1-\gamma}$. We term this as \textbf{C2: objective mismatch}
challenge in MBRL. The model fitting objective of minimizing
prediction error is not a good approximation for
terms~\eqref{term:dlm_2} and~\eqref{term:dlm_3}, which are able to
capture the relative importance of transition $(s,a)$ in policy
computation through the value of the resulting successor.
% This causes the challenge of
% \textbf{C2: objective mismatch} as the resulting TV objective is not
% necessarily correlated
% \aarti{poor wording, need to change}
% with the
% original terms~\pref{term:dlm_2}
% and~\ref{term:dlm_3}.
For example in~\pref{fig:toy}, $M^{\mathsf{bad}}$ has lower
prediction error than $M^{\mathsf{good}}$ even though it makes a
mistake at the crucial state $A$ where the value difference between
predicted and true successor is high, while $M^{\mathsf{good}}$ is
better according to objective in~\pref{lem:double}.

%%% Local Variables:
%%% mode: latex
%%% TeX-master: "../example_paper"
%%% End:

\section{Performance Difference via Advantage in Model}\label{sec:simulation}

% The last section shows two disadvantages of the previous
% decomposition, namely the computation efficiency issue and the
% objective mismatch issue. To overcome the problem, we introduce a new
% decomposition of the same objective, and we name the result the
% simulation performance difference lemma. As we will see later on, the
% simulation performance difference lemma resolves these two issues and
% results in an objective is more feasible to optimize.

To overcome the challenges \textbf{C1} and \textbf{C2} presented in
the previous section, let us revisit the performance difference
in~\pref{lem:double} and introduce a \textit{new decomposition} for it that
results in a unified objective which is more
feasible to optimize. This decomposition is the primary contribution
of this paper and we name it as \textit{Performance
  Difference via Advantage in Model} (\pdam{}). The detailed proof can
be found in~\pref{app:proofs}.

\begin{lemma}[Performance Difference via Advantage in Model]\label{lem:simulation-pdl}
    Given any start state distribution $\omega$, policies $\pihat,
    \pistar$, and transition functions $\Mhat, \Mstar$ we have:
    \begin{align}
        (1-&\gamma)[J_{\Mstar}^\omega(\pihat) -
          J_{\Mstar}^\omega(\pistarc)] = \nonumber\\
             &\underbrace{\expec_{s \sim d_{\omega, \pistarc}}[
          V^{\pihatc}_{\Mhat}(s) - \expec_{a \sim
          \pistarc(s)}[Q^{\pihatc}_{\Mhat}(s,
               a)]]}_{\text{Disadvantage on states visited by expert}} \label{term:policy}\\
        &+\underbrace{\gexpec_{\substack{(s, a) \sim D_{\omega,
          \pihatc} \\ \sprime \sim \Mstar(s, a)}}[
          V_{\Mhat}^{\pihatc}(s') - \expec_{s'' \sim \Mhat(s,
          a)}[V_{\Mhat}^{\pihatc}(s'')]]}_{\text{Value difference on
      states visited by learned policy}} \label{term:model_learned}\\
        &+\underbrace{\gexpec_{\substack{(s, a) \sim D_{\omega,
          \pistarc} \\ \sprime \sim \Mstar(s, a)}}[\expec_{s'' \sim \Mhat(s,
          a)}[V^{\pihatc}_{\Mhat}(s'')] -
          V^{\pihatc}_{\Mhat}(s')]}_{\text{Value difference on states
      visited by expert}}\label{term:model_expert}
    \end{align}
\end{lemma}

The above lemma presents a novel unified objective for joint model and
policy learning. We can make a few important remarks. First, bounding
term~\eqref{term:policy} does not require computing the optimal policy
in the learned model $\Mhat$, unlike term~\eqref{term:dlm_1}.
Instead, we need a policy that has small
``disadvantage'' over the optimal policy in the learned model at
states sampled along $\pistar$ trajectory in $\Mstar$. This
disadvantage term was popularly used as an objective for policy search in
several model-free
works~\cite{kakade2002approximate,bagnell2003psdp,ross2014aggrevate}. Given access to
an exploration distribution $\nu$ that covers $D_{\omega, \pistar}$,
computing such a policy requires computation that is polynomial in the
effective task horizon~\cite{kakade2003sample}, compared
to~\eqref{eq:oc} which can require
computation that is exponential in the task horizon. In other words, by being
``lazy'' in the policy computation step we can solve challenge \textbf{C1} while
still optimizing the performance difference between
$\pihat$ and $\pistar$.

Second, while \pdam{} looks similar
to~\pref{lem:double} for the model fitting
terms~\eqref{term:model_learned} and~\eqref{term:model_expert}, there
is one crucial
difference: we only need $V_{\Mhat}^{\pihat}$ in the new objective,
which is feasible to compute using any policy evaluation method in the
learned model~\cite{sutton2018reinforcement}. On the other hand,~\pref{lem:double}
required access to $V_{\Mhat}^{\pistar}$ where $\pistar$ is
unknown and hence, we had to upper bound the objective in
Corollary~\ref{cor:tv} using model prediction error. Optimizing
prediction error requires the learned model to capture dynamics
everywhere equally well. However, the unified objective
\pdam{} offers a ``lazy'' alternative for model fitting
that focuses only on transitions that are critical for policy
computation by being value-aware, solving challenge \textbf{C2}.

%%% Local Variables:
%%% mode: latex
%%% TeX-master: "../example_paper"
%%% End:

\section{\ouralg{}: Lazy Model-based Policy Search}
\label{sec:alg_mle}

\begin{algorithm}[t]
\caption{Minimize Disadvantage \textbf{ComputePolicy}($\Mhat_t$)}
\begin{algorithmic}[1]
\REQUIRE Exploration distribution $\nu$, learned model $\Mhat_t$,
policy class $\Pi$.
\STATE Find $\pihat_t \in \Pi$ using cost-sensitive
classification on states sampled from $\nu$
in $\Mhat_t$~\cite{ross2014aggrevate} such that
\begin{align}\label{eq:policy}
    \expec_{s \sim \nu}\left[V^{\pihat_t}_{\Mhat_t}(s) - \min_{a \in
  \Acal}[Q^{\pihat_t}_{\Mhat_t}(s,a)]\right] \leq \epsilon_{po}
\end{align}
\STATE \textbf{Return} $\pihat_t$
% \STATE Initialize dataset $\Dcal \leftarrow \phi$, random policy
% $\pihat_1 \in \Pi$
% \FOR{$i=1$ to $N$}
% \STATE Sample $m$ data points by sampling a state $s$ from $\nu$,
% executing a random action $a$
% \ENDFOR
\end{algorithmic}\label{alg:policy}
\end{algorithm}

% To see how to solve the \textbf{computation efficiency} problem, now
% we introduce our first algorithm that utilizes
% \pref{lem:simulation-pdl}. Using a similar argument as in
% Corollary~\pref{cor:tv}, we can bound the simulation performance
% difference lemma using a Total variation bound as follows:
In this section, we present our first algorithm \ouralg{} that uses
the unified objective \pdam{} to solve
challenge \textbf{C1}.~\pref{alg:policy} performs policy optimization
along the exploration distribution similar to previous policy search
methods~\cite{kakade2002approximate,bagnell2003psdp,ross2014aggrevate}.
% \footnote{Note that \citet{ross2014aggrevate} is slightly different
% as it only works under expert distribution.}.
Note that \pref{eq:policy} requires performing one-step cost-sensitive
classification only at states sampled from the exploration
distribution $\nu$ making~\pref{alg:policy} computationally
efficient. On the other hand, optimal planning requires minimizing disadvantage at all
states under the learned policy $d_{\omega, \pihat_t}$\footnote{To see this, apply the performance difference lemma~\cite{kakade2002approximate} to term~\eqref{term:dlm_1} in~\pref{lem:double}.} which changes
as the learned policy is updated leading to the exponential dependence
on horizon.
%\jab{This phrase needs work. Compute at all the states what?}

% The
% result of~\pref{alg:policy}
% is policy $\pihat$ that satisfies,
% \begin{equation}
%   \label{eq:po_error}
%   \expec_{s \sim \nu}\left[V^{\pihat}_{\Mhat}(s) - \min_{a \in
%   \Acal}[Q^{\pihat}_{\Mhat}(s,a)]\right] \leq \epsilon_{dis}
% \end{equation}
To understand how~\pref{alg:policy} can help optimize \pdam{}, we use
H\"older's inequality on terms~\eqref{term:model_learned}
and~\eqref{term:model_expert} to bound \pdam{} as, \looseness=-1
\begin{corollary}\label{cor:tv_new}
        For any start state distribution $\omega$, transition
        functions $\Mhat, \Mstar$, and policies $\pihat,
        \pistar$ such that $\pihat$ satisfies~\pref{eq:policy}, we have,
\begin{align}
    (1-\gamma)&[J_{\Mstar}^\omega(\pihat) -
                J_{\Mstar}^\omega(\pistar)] \leq \epsilon_{po}
                                \nonumber\\
                                &+\gamma \hat V_{\max}\expec_{(s,
          a) \sim D_{\omega, \pihat}}\|\Mhat(s, a) - \Mstar(s, a)\|_1
          \nonumber\\
        &+ \gamma \hat V_{\max}\expec_{(s, a) \sim
          D_{\omega, \pistar}}\|\Mhat(s, a) - \Mstar(s, a)\|_1
          \nonumber
\end{align}
where $\hat{V}_{\max} = \|V_{\Mhat}^{\pihat}\|_\infty$.
\end{corollary}

% \pref{lem:tv_new} suggests that we could use a similar algorithm to
% Sysid by a reduction to no-regret online algorithm.\footnote{The
%   justification of the reduction to no-regret is covered in
%   \citet{ross2012agnostic} in detail so we will skip this part, but we
%   provide a toy scenario in \pref{sec:hedge} in the appendix. } Note
% that the 3rd term is not the planning error, instead it is the
% expected advantage of the learned policy \textit{in the model} along
% the expert policy's trajectory in \textit{the real world}. Driving
% this term down to $\epsilon$ only requires computation that is
% polynomial in horizon, whereas driving down the planning error to
% $\epsilon$ could potentially take computation that is exponential in
% horizon. Thus on the algorithm level, the only modification we need to
% make on the Sysid is to replace the optimal control stage to a policy
% search stage to find any policy such that the third term is small
% (non-positive), i.e., we run \pref{alg:meta} with \pref{alg:model_mle}
% as model fitting subprotocol and \pref{alg:policy} as policy training
% subprotocol (let us denote this algorithm as \pref{alg:policy} for
% convenience).

Corollary~\pref{cor:tv_new} indicates a simple modification to
existing MBRL methods where we use MLE-based~\pref{alg:model_mle} for
model fitting and~\pref{alg:policy} for policy computation in the
framework of~\pref{alg:meta}. We refer to this new algorithm as
\ouralg. Note that by upper bounding the model fitting terms
in~\pref{lem:simulation-pdl} using model prediction error,
\ouralg{} is only able to solve challenge \textbf{C1} but not
\textbf{C2}.

% \jab{Suggest something like "as with offline RL approaches" regarding coverage. Also suggest
% italic-sying "cover" in it's first use. Coverage coefficient might need another ref older than agnostic as well.
% The footnote isn't bad actually.}

Our algorithm only requires an exploration distribution $\nu$ that
\textit{covers} the expert policy state-action distribution
$D_{\omega, \pistar}$ as described in~\pref{sec:setup}. To capture how
well $\nu$ covers $D_{\omega, \pistar}$, we define the coverage
coefficient
$\Ccal = \sup_{s,a} \frac{D_{\omega,
    \pi^\ast}(s,a)}{\nu(s,a)}$ similar to~\citet{ross2012agnostic}\footnote{Note that there are also
  coverage notions in the offline and hybrid RL
  setting~\cite{xie2021bellman,song2022hybrid}, but their coverage
  coefficient is not directly applicable in the model-based setting.}.
We now present the regret bound for \ouralg{} using this coverage
coefficient and proof can be found in \pref{app:proofs}.
% \aarti{replace last sentence with something like'We now present
% regret bounds on POP-DIS using the coverage coefficient. Proof of
% the Theorem below is given in Appendix' (dont call our Corr A.2 here
% - instead in proof of Thm 5.1.)}

\begin{theorem}\label{thm:policy}
  Let $\{\hat \pi_t\}^T_{t=1}$ be the sequence of returned policies of
  \ouralg. We have:~\footnote{We use $\tilde O$ to omit logarithmic
    dependencies on terms.}
    \begin{align*}
      \frac{1}{T}\sum_{t=1}^T &J_{\Mstar}^\omega(\pihat_t) - J_{\Mstar}^\omega(\pistar) \leq \\
                              &\tilde O\left(\Ccal \epsilon_{po} + \frac{\Ccal \hat
                                V_{\max}}{1-\gamma}\left(\sqrt{\epsilon^{KL}_{model}} +
                                \frac{1}{\sqrt{T}}\right)\right),
    \end{align*}
    where $\hat V_{\max} = \|V^{\hat \pi}_{\hat M}\|_{\infty}$,
    $\Ccal = \sup_{s,a} \frac{D_{\omega, \pi^\ast}(s,a)}{\nu(s,a)}$ is
    the coverage coefficient, $\epsilon_{po}$ is the policy advantage error,
    and
    $\epsilon^{KL}_{model} = \min_{M \in \Mcal}\expec_{s,a \sim \bar
      \Dcal_T} \mathsf{KL}(M(s,a), M^{\ast}(s,a))$ is the agnostic
    model error\footnote{Here we use $\mathsf{KL}(\cdot,\cdot)$ to
      denote the Kullback–Leibler divergence and  denote the
      training data distribution as
      $\bar \Dcal_T = \frac{1}{T}\sum_{t=1}^T\Dcal_t$}.
    %\end{align*}
\end{theorem}

In comparison, \citet{ross2012agnostic}'s regret bound is
\begin{small}
\begin{align*}
  \tilde O\left(\epsilon_{oc} + \frac{\Ccal\max\{V_{\max}, \hat
  V_{\max}\}}{1-\gamma}\left( \sqrt{\epsilon^{KL}_{model}}+\frac{1}{\sqrt{T}}
  \right)  \right),
\end{align*}
\end{small}
where $V_{\max} = \|V^{\pi^\ast}_{\hat M}\|_\infty$ is difficult to
optimize as $\pistar$ is unknown, and can be as large as the effective
horizon $\frac{1}{1-\gamma}$. On the other hand, our bound only has
$\hat{V}_{\max} = \|V^{\pihat}_{\Mhat}\|_\infty$ which is optimized
in~\pref{alg:policy} for states that are sampled from $\nu$. Thus, we
expect that there exists cases where $\hat{V}_{\max}$ is much smaller
when compared to $V_{\max}$\footnote{Note that in the worst case,
  $\hat{V}_{\max}$ can also be as large as
  $\frac{1}{1-\gamma}$. However, we show an experiment
  in~\pref{sec:experiment} where $\hat{V}_{\max}$ is smaller than
  $V_{\max}$ in practice.}, leading to a tighter regret bound
in~\pref{thm:policy}. Hence, while \ouralg{} primarily solves
challenge \textbf{C1} lending computational gains, we also observe
statistical gains in practice as we show in our experiments
in~\pref{sec:experiment}. Another crucial difference is that the
coverage coefficient $\Ccal$ shows up for the policy optimization error in
\ouralg{} regret bound which suggests that \ouralg{} is relatively
more sensitive to the quality of exploration distribution $\nu$. We
show an example of this in~\pref{app:failure}.

\begin{algorithm}[t]
\caption{Moment Matching~\textbf{FitModel}($\Dcal_t, \{\ell_i\}_{i=1}^{t-1}$)}
\begin{algorithmic}[1]
\REQUIRE Data $\Dcal_t$, model class $\Mcal$, previous losses
$\{\ell_i\}_{i=1}^{t-1}$, Strongly convex regularizer $\mathcal{R}$
\STATE Define loss
\begin{align}\label{eq:absolute}
  \ell_t(M) = \expec_{(s, a, s') \sim
  \Dcal_t}\left|V^{\pihat_t}_{\Mhat_t}(s') - \expec_{s'' \sim M(s, a)}[V^{\pihat_t}_{\Mhat_t}(s'')]\right|
\end{align}
\STATE Compute model $\Mhat_{t+1}$ using an online no-regret
algorithm, such as FTRL~\cite{hazan2019oco},
\begin{align}\label{eq:ftl}
    \hat M_{t+1} \leftarrow \argmin_{M \in \Mcal} \sum_{\tau=1}^t  \ell_\tau(M) + \mathcal{R}(M).
\end{align}
\STATE \textbf{Return} $\hat M_{t+1}$
\end{algorithmic}\label{alg:model_mm}
\end{algorithm}

\section{\ouralgmm{}: Lazy Model-based Policy Search via Value Moment
  Matching}\label{sec:moment}

%\todo[inline]{Need to add the algorithm for the moment matching
%approach along with its no-regret guarantees}
% The previous section introduces a new algorithm that achieves a
% tighter regret bound than \citet{ross2012agnostic} due to the new
% decomposition. However, the algorithm operates on the result of
% Corollary~\pref{cor:tv_new}, which is still a relatively loose bound
% due to Holder's inequality, indicating \textbf{C2}. Thus one natural
% way is to design an algorithm that directly operates on
% \pref{lem:simulation-pdl}, which is equality. Using the same data
% aggregation argument as \pref{alg:model_mle}%\aarti{as?}
% , we define the loss for a model $M$ at round $t$, $\ell_t(M)$ as in
% \pref{eq:squared}, whose square root upper bounds %\aarti{not exactly?
%                                 %need square root?}
% \pref{term:model_learned} and \pref{term:model_expert} by Jensen's
% inequality, to obtain the no-regret result.
The previous section introduced an algorithm \ouralg{} that, while
being more computationally efficient than existing MBRL methods, did
not reap the full benefits of our proposed unified objective
\pdam{}. \ouralg{} optimizes the objective presented in
Corollary~\ref{cor:tv_new} which is a weak upper bound of the unified
objective (as explained in~\pref{sec:challenges}.) A natural question
would be to ask whether there exists an algorithm that directly
optimizes the unified objective without constructing a weak upper
bound. We present such an algorithm in this section.  The key idea is
to formulate it as a moment matching
problem~\cite{sun2019ilfo,swamy2021moments} where our learned model is
matching the true dynamics in expectation using value moments.

% \todo[inline]{Why is the squared objective is not yet another weak
% upper bound on the SimPDL?}

\pref{alg:model_mm} minimizes the value moment difference between true
and predicted successors on a given dataset of transitions. The objective $\ell_t$ in~\eqref{eq:absolute} upper bounds the terms~\eqref{term:model_learned} and~\eqref{term:model_expert}, and using a Follow-the-Regularized-Leader (FTRL) approach with a strongly convex regularizer $\mathcal{R}(M)$ allows us to achieve a no-regret guarantee.
% Note that
% we use the objective $\ell_t$ in~\eqref{eq:squared} rather than the
% original linear objective in~\eqref{term:model_learned}
% and~\eqref{term:model_expert} as it is strongly convex in $M$ and is
% \jab{Wait.... I don't think it's strongly convex in M at all.}
% thus easier to optimize using no-regret online learning algorithms.
% \jab{Given that, we should be more thoughtful about reasons.} We
% also present a similar algorithm that optimizes the original linear
% objective and present its regret bound in~\pref{app:signed}.
\iffalse
  The objective $\ell_t$ also contains a value moment which is the
  value of the given policy $\pihat_t$ in the given model
  $\Mhat_t$. Hence, to solve~\eqref{eq:ftl} using a batch algorithm,
  like gradient descent, would require not only aggregating data
  $\Dcal_t$ but also value functions $V^{\pihat_t}_{\Mhat_t}$ over all
  previous iterations.  \fi
%$t=1,\ldots,T$.
We refer to the new MBRL algorithm that uses value moment matching
based~\pref{alg:model_mm} for model fitting and~\pref{alg:policy} for
policy computation in the framework of~\pref{alg:meta} as \ouralgmm{}.

% Note that \pref{eq:squared} gives us strong convexity
% %\aarti{these are not two advantages - the
% %  practical friendliness is because of strong convexity. Second it
% %  leads to tighter regret bounds as shown next.}
% comparing to the original linear loss in \pref{term:model_learned} and \pref{term:model_expert},
% and thus we can still
% apply any no-regret online algorithm. We describe the value moment
% matching algorithm in \pref{alg:model_mm}. We denote our algorithm
% that runs \pref{alg:meta} with \pref{alg:model_mm} for model fitting
% and \pref{alg:policy} for policy training as \ouralgmm, and below we
% present its regret bound:
\ouralgmm{}, by virtue of optimizing \pdam{}, does not suffer
challenge \textbf{C2} as the model fitting objective helps learn
models that are useful for policy computation by focusing on critical
states where any mistake in action can lead to large value
differences. Going back to the example in~\pref{fig:toy}, \ouralgmm{}
would pick $M^{\mathsf{good}}$ over $M^{\mathsf{bad}}$ as the latter
incurs high loss $\ell_t$~\eqref{eq:absolute} at state $A$. Thus,
\ouralgmm{} solves both challenges \textbf{C1} and \textbf{C2}. This
results in improved statistical efficiency as indicated by its tighter
regret bound:
\begin{theorem}\label{thm:mmregret}
  Let $\{\hat \pi_t\}^T_{t=1}$ be the sequence of returned policies of
  \ouralgmm, we have:
    \begin{align*}
      \frac{1}{T}\sum_{t=1}^T J_{\Mstar}^\omega(\pihat_t)& -
                                                           J_{\Mstar}^\omega(\pistar)
                                                           \leq  \\
                                                         &\tilde O\left(\Ccal \epsilon_{po} +  \frac{
                                                           \Ccal}{1-\gamma} \left(\epsilon^{mm}_{model} +
                                                           \frac{1}{\sqrt{T}}\right)\right),
    \end{align*}
    where
    $\epsilon^{mm}_{model} = \min_{M \in \Mcal}
    \frac{1}{T}\sum_{t=1}^T \ell_t(M)$ is the agnostic model error,
    $\epsilon_{po}$ is the policy advantage error, and
    $\Ccal = \sup_{s,a} \frac{D_{\omega, \pi^\ast}(s,a)}{\nu(s,a)}$ is
    the coverage coefficient.
   %\todo[inline]{@Yuda: What is $V$ in the definition of
    %  $\epsilon^{mm}_{model}$? We need to remove dependence on $\max_t$}
\end{theorem}

Comparing the above regret bound with that of \ouralg{}
in~\pref{thm:policy}, we have improved the bound by getting rid of the
dependency on $\hat{V}_{\max}$\footnote{Note that
  $\epsilon_{model}^{mm}$ implicitly depends on the value
  function. However, if the model class $\Mcal$ is rich enough, we can
  expect this error to be small. For example, if $\Mcal$ contains
  $\Mstar$ then this error goes to zero resulting in a regret bound
  with no dependence on $\hat{V}_{\max}$, whereas the regret bound
  in~\pref{thm:policy} still retains a dependence on $\hat{V}_{\max}$
  even in the realizable case.} which as stated in~\pref{sec:alg_mle},
can be as large as $\frac{1}{1-\gamma}$ in the worst case.
% We remark that comparing with \pref{thm:policy}, \pref{thm:mmregret}
% improves the regret bound by getting rid of the dependency on the
% function values \footnote{In fact, the agnostic model error
% implicitly depends on the value function in this case, but
% all functions in the regression problem are known,
% the error is estimated only under certain distributions,
% and thus this error will be small if the model class
% if rich enough.}. Also note that although using squared loss introduced
% yet another upper bound on the decomposition, the regret rate
% is using signed version of the loss still result in the same rate of regret.
% \todo[inline]{YS will add another section in the appendix to discuss this.}
%\aarti{agnostic model error still depends on the value
%  function, but it will be smaller because not judging model
%  everywhere}.
% Note as discussed in \pref{sec:related}, this moment matching
% version of the algorithm shares similar intuition with some previous
% works such as value targeted regression \citep{ayoub2020model} and
% value equivalence principle \citep{grimm2020value}. These previous
% works also argue for finding a model that minimizes the value
% differences, but
% \todo[inline]{finish listing the differences}

Despite the statistical advantages of \ouralgmm{}, its practical
implementation is difficult as estimating the loss
$\ell_t$~\eqref{eq:absolute} requires evaluating the policy in the
learned model which can only be approximated in large MDPs. A similar
difficulty was also observed in previous
works~\cite{ayoub2020model,grimm2020value,farahmand2017value} that
proposed value-aware objectives for model fitting.
% \jab{Not sure I follow this. I think we need to work through how sample-hard it is
% to evaluate these. We need the model value function, but how accurately do we need it?
% Good in l_1 or l_$\inf$?}
We highlight a few
scenarios in which~\pref{alg:model_mm} is practically
realizable. First, for finite MDPs we can evaluate the policy exactly
avoiding this difficulty. Second, in MDPs such as linear dynamical
systems with quadratic costs where we can compute the value of policy
in closed form, we can estimate $\ell_t$ and use a gradient-based
optimization method to find the best model in the model class in the
optimization problem~\eqref{eq:ftl}. We demonstrate both of these
scenarios in our experiments in~\pref{sec:experiment}.
It is also important to note that solving~\eqref{eq:ftl} using a batch
algorithm, like gradient descent, would require aggregating both data
$\Dcal_t$ and value functions $\hat{V}^{\pihat_t}_{\Mhat_t}$ across
iterations in~\pref{alg:meta}. This is advantageous as our approach does not require us to compute gradients through how changes in model $M$ affect the value estimates used in $\ell_t(M)$ (see~\eqref{eq:absolute}) which can be very difficult to compute.
For completeness, we also give a finite sample analysis for
\ouralgmm{} in Appendix~\ref{sec:finite}. Note that we skip the finite
sample analysis for \ouralg{},
because the analysis takes the same form as \citet{ross2012agnostic}.

%Finally, even with our derivation for the gradient estimation, as
%already pointed out in the previous works
%\citep{ayoub2020model,grimm2020value}, unless we can have easy access
%to the true value function under the model (for example, the value
%function can be calculated closed-formly given the dynamics and
%policy), this moment-matching model update is generally very hard in
%practice, as it requires an accurate value function for the
%regression problem. Nevertheless, we do showcase results in
%\pref{sec:widetree} to demonstrate the advantage of the moment
%matching version when it is feasible to run.

%%% Local Variables:
%%% mode: latex
%%% TeX-master: "../example_paper"
%%% End:

\begin{figure*}[t]
  \centering
  \includegraphics[width=0.19\linewidth]{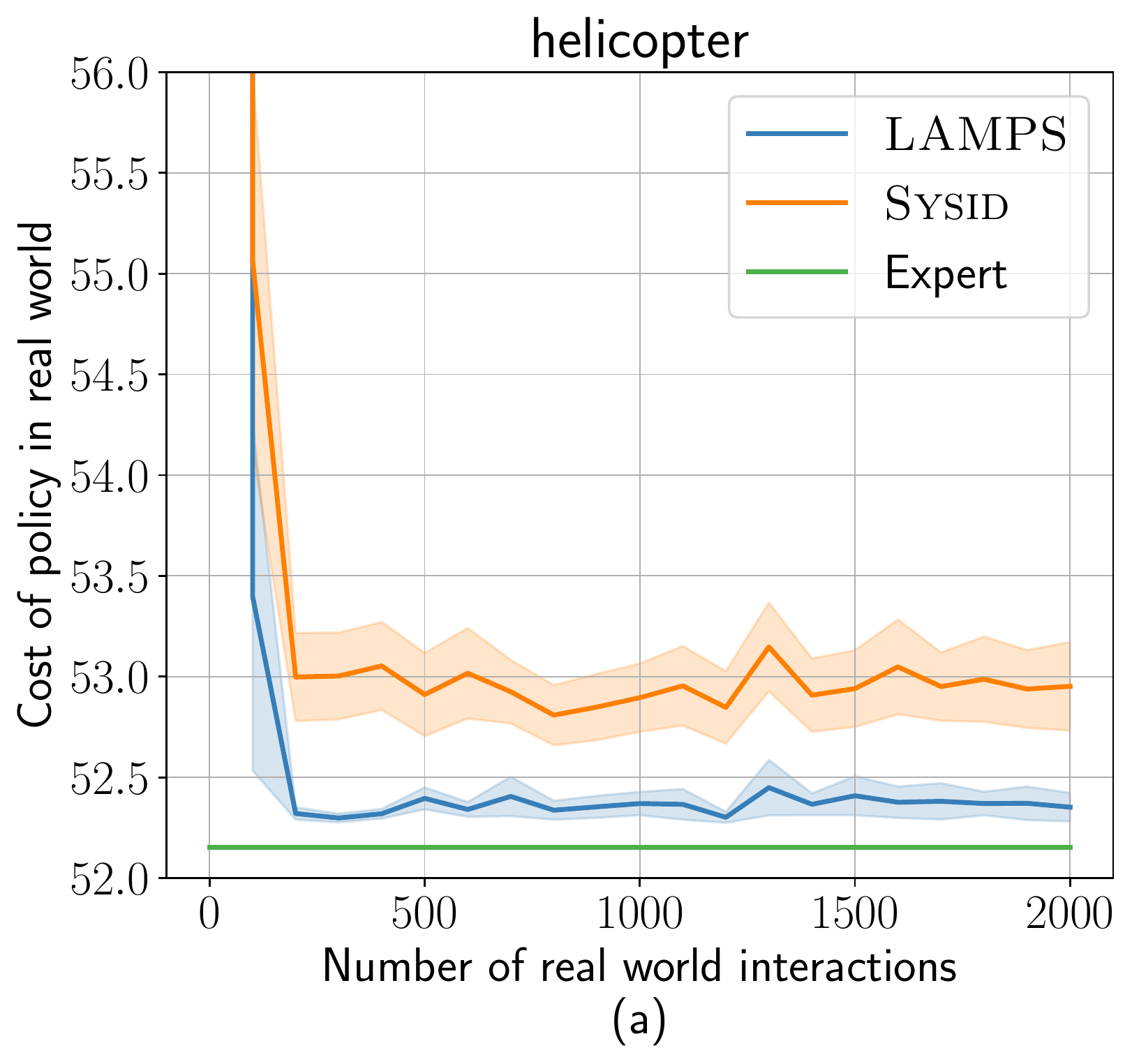}
  \includegraphics[width=0.19\linewidth]{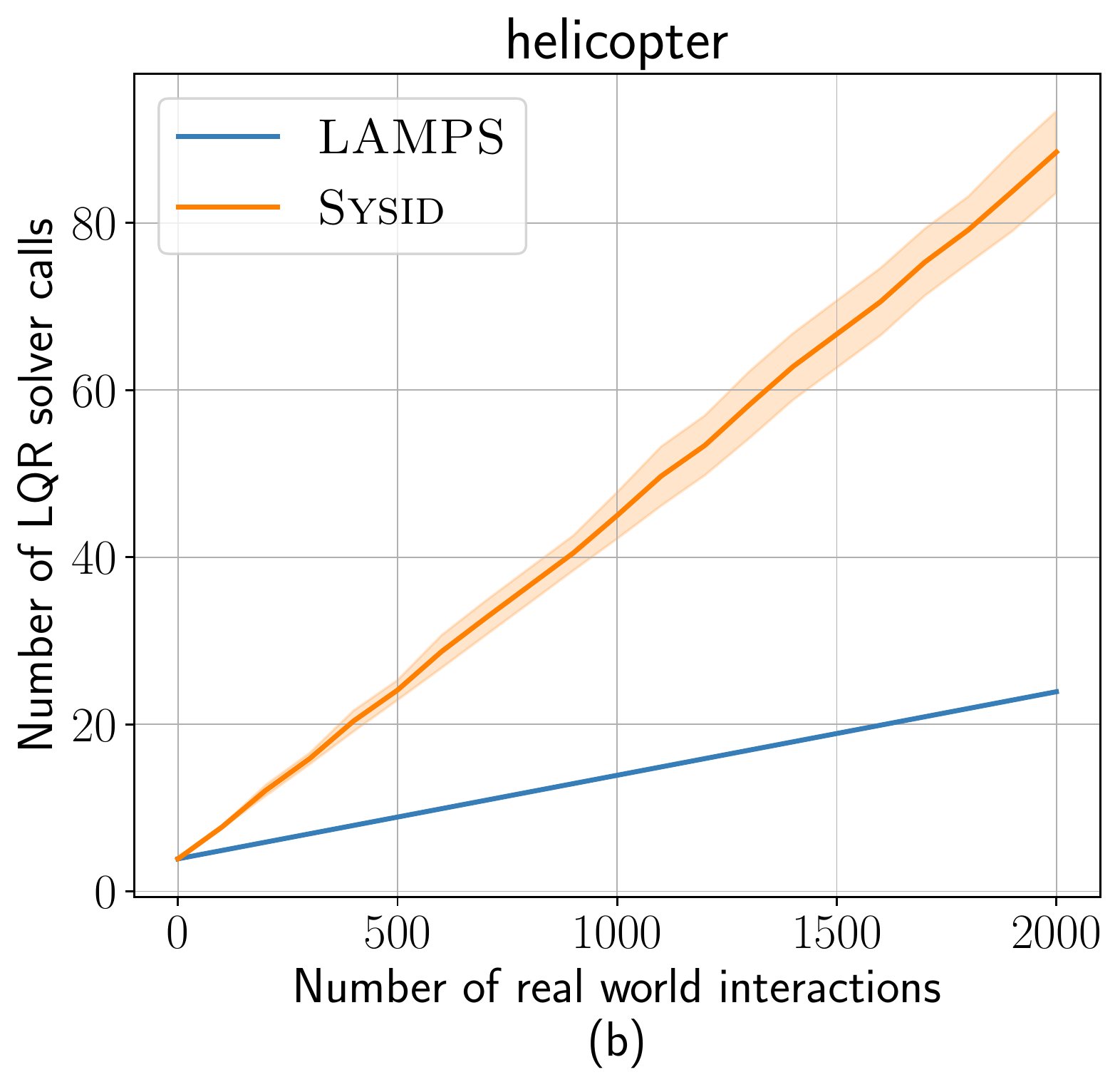}
  \includegraphics[width=0.19\linewidth]{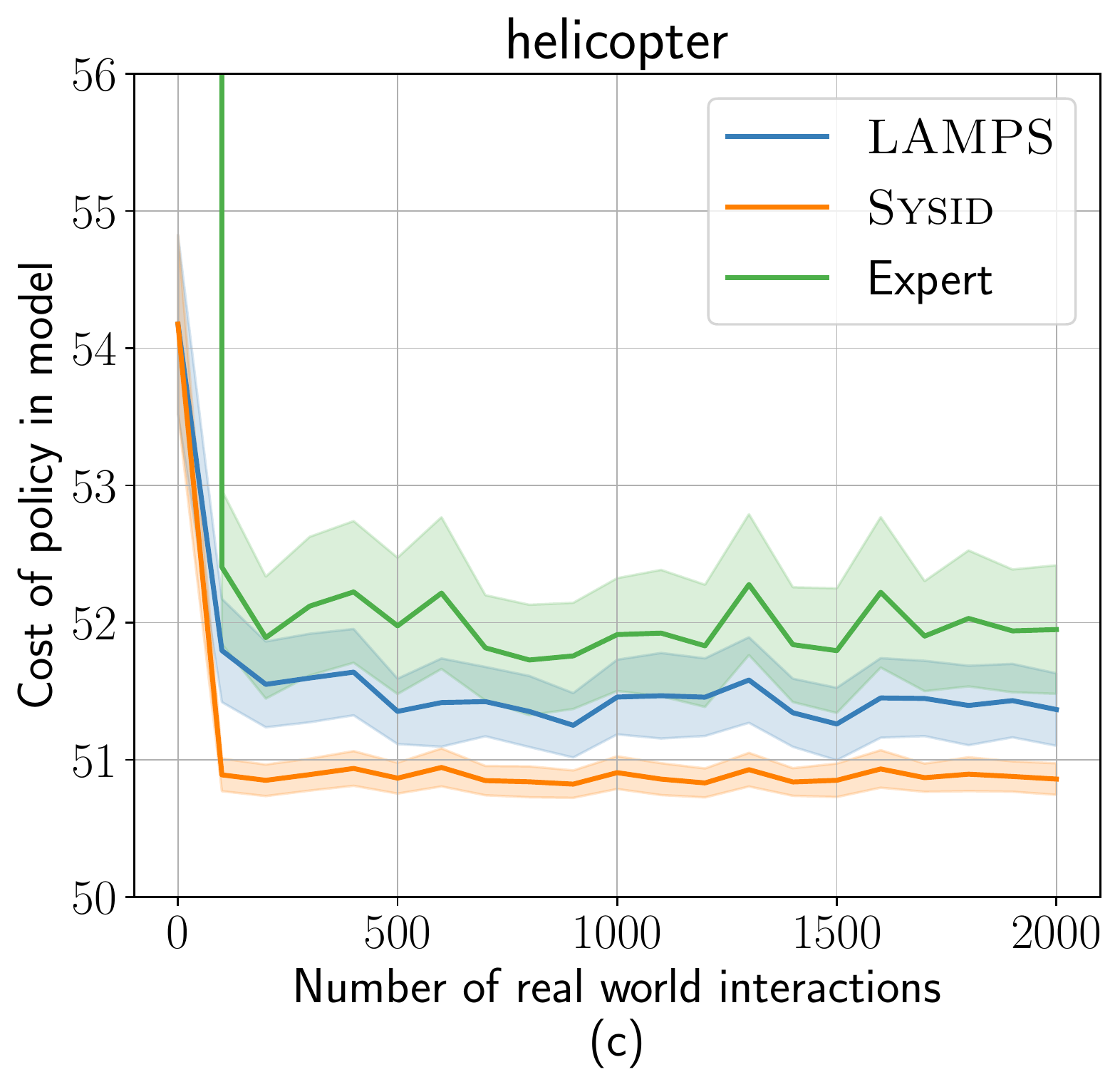}
  \includegraphics[width=0.19\linewidth]{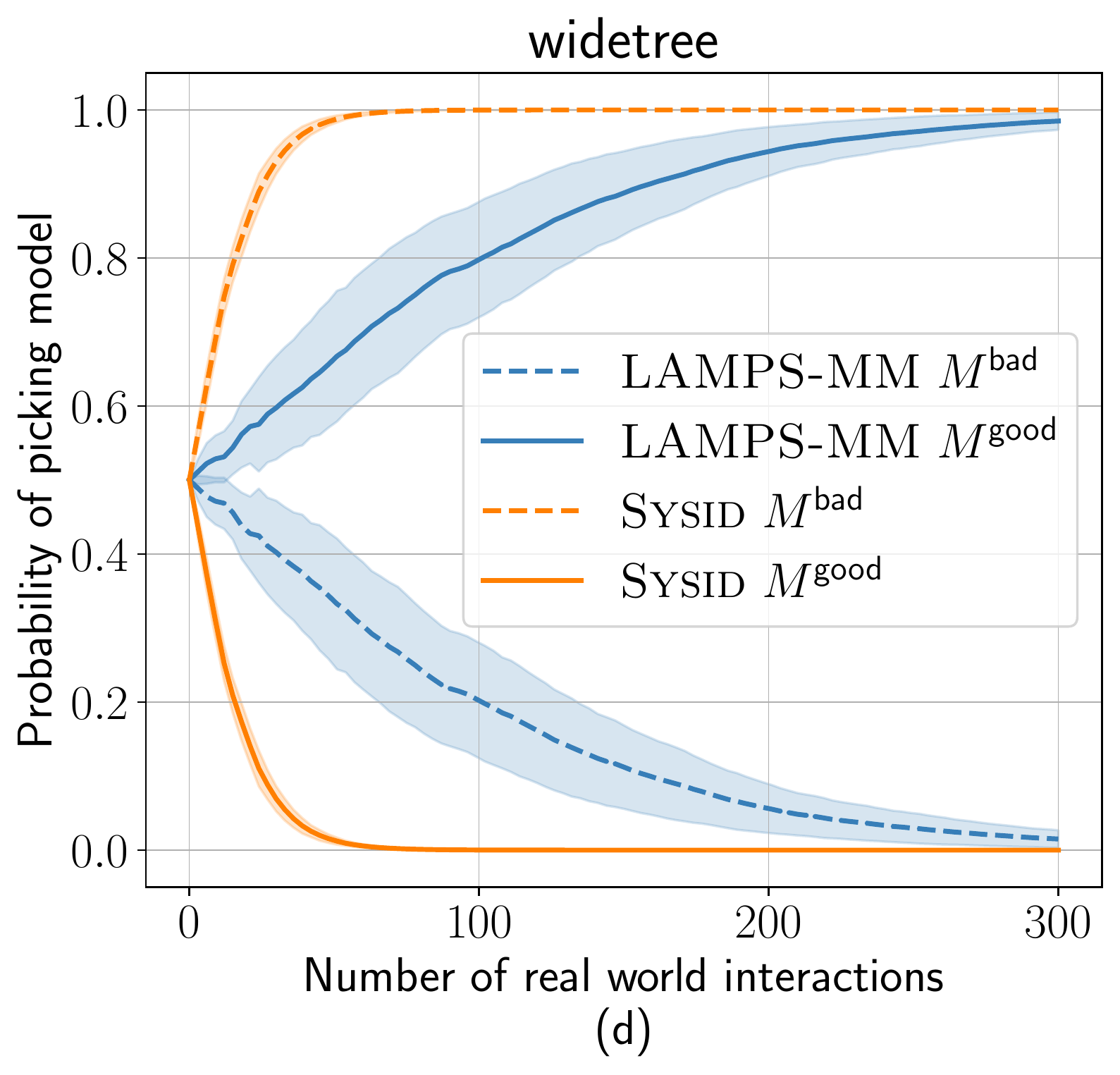}
  \includegraphics[width=0.21\linewidth]{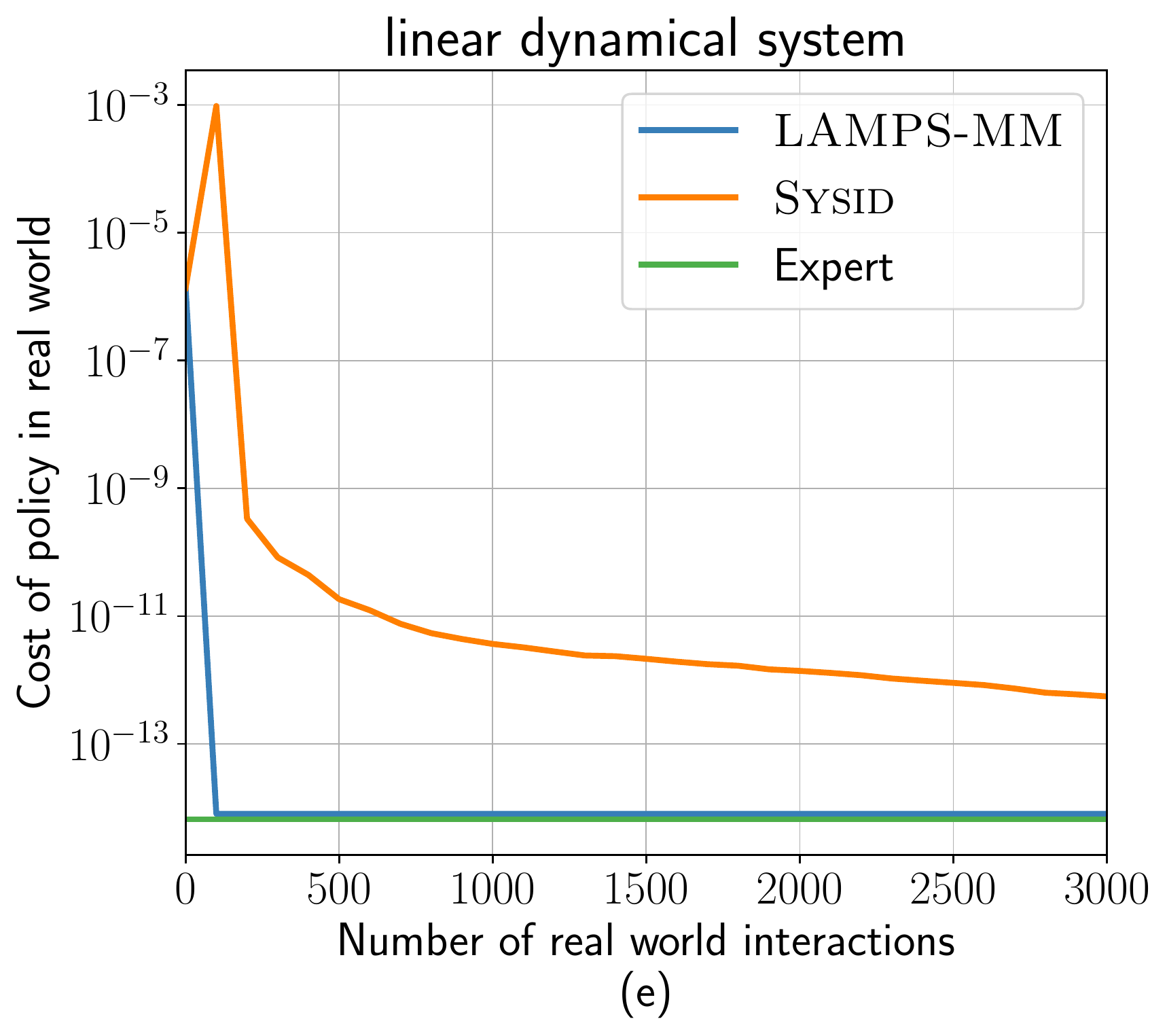}
  \caption{Results on our Helicopter, WideTree, and Linear Dynamical
    System (LDS) Domains. All experiments are done using $10$ random seeds
    and the solid lines show the mean while shaded area depicts
    standard error. For helicopter and LDS experiments at each iteration, we add
    $100$ samples, that are sampled with equal probability from
    exploration distribution and learned policy rollouts, to the dataset.}
  \label{fig:small_mdp}
\end{figure*}
\section{Experiments}\label{sec:experiment}
In this section, we present experiments that test our proposed
algorithms against baselines across five varied domains.
% (1) the nonlinear
% helicopter domain used in~\citet{ross2012agnostic}, (2) the finite
% widetree domain used in~\citet{ayoub2020model} which highlights the
% advantages of \ouralgmm{}, (3) a linear dynamical system where the
% value function can be computed in closed form, (4) the standard dense
% reward mujoco benchmarks~\cite{brockman2016openai} where we compare with
% recent MBRL baselines and their variants, and finally (5) a sparse
% reward maze environment used in~\citet{fu2020d4rl}.
For baselines, we compare with~\citet{ross2012agnostic} that uses MLE
for model fitting and optimal planning for policy computation, and
call this baseline as \sysid{}. In addition to this, we also use
\mbpo{}~\cite{janner2019trust} and design variants of it that utilize
exploration distribution, similar to \sysid{}, to ensure a fair
comparison with our proposed algorithms. We defer details of our
experiment setup, cost functions, and baseline implementation
to~\pref{app:exp}\footnote{Code for all our experiments can be found
  at \url{https://github.com/vvanirudh/LAMPS-MBRL}.}.
%\textbf{baselines:} ours, MBPO-SysId, MBPO-SysId (2x), MBPO
\subsection{Helicopter}\label{sec:helicopter}
In this domain from~\citet{ross2012agnostic}, we compare \ouralg{}
with \sysid{}. The objective of the task is for the helicopter to
track a desired trajectory with unknown dynamics under the presence of
noise. For both approaches, we use an exploration distribution $\nu$
that samples from the desired trajectory. For optimal planning in
\sysid{}, we run iLQR~\cite{li2004ilqr} until convergence, and to
implement~\pref{alg:policy} for \ouralg{} we run a single iteration of
iLQR where the forward pass is replaced with the desired trajectory
and we run a single LQR backward pass on it to compute the policy. For
detailed explanation on the dynamics, cost function, and
implementations of \sysid{} and \ouralg{}, refer
to~\pref{app:helicopter}.

\pref{fig:small_mdp}(a) shows that \ouralg{} can learn a better
policy than \sysid{} given the same amount of real world data and the
same exploration distribution, indicating statistical gains. To test our
hypothesis that this is due to the tighter regret bound for \ouralg{}
as $\hat{V}_{\max} < V_{\max}$,~\pref{fig:small_mdp}(c) shows
how the learned policy $\pihat$ performs in the learned model $\Mhat$
for both approaches, and the expert's performance in $\Mhat$. Observe
that both \ouralg{} and \sysid{} are able to optimize
$V^{\pihat}_{\Mhat}$ but the expert's performance
$V_{\Mhat}^{\pistar}$ is not optimized as well leading to a weaker
regret bound for \sysid{} when compared to
\ouralg{}.~\pref{fig:small_mdp}(b) shows the computational
benefits of \ouralg{} where we plot the number of LQR solver calls,
the most expensive operation, made by each approach and we can observe
that by only optimizing on the exploration distribution \ouralg{}
significantly outperforms \sysid{} in the amount of computation used.

\subsection{WideTree}\label{sec:widetree}
We use a variant of the finite MDP domain introduced
in~\citet{ayoub2020model} that is very similar to~\pref{fig:toy} with
$H=3$. The model class consists of two models: $M^{\mathsf{good}}$ and
$M^{\mathsf{bad}}$ as described in~\pref{fig:toy}. For detailed
explanation of the dynamics, refer to ~\pref{fig:widetree}
in~\pref{app:widetree}. We compare \ouralgmm{} and \sysid{}. To
compute the best model to pick at every iteration given the data so
far, we use Hedge~\cite{freund1997hedge} to update the discrete
probability distribution over the two models.~\pref{fig:small_mdp}(d)
shows how the distribution evolves when using MLE-based model fitting
loss in \sysid{} and value moment matching loss in \ouralgmm{}. As
$M^{\mathsf{bad}}$ matches true dynamics everywhere except at the
root, \sysid{} collapses to a distribution that picks the bad model
over a good model always, while \ouralgmm{} which reasons about the
usefulness of transitions in computing good policies converges to a
distribution that picks the good model more often.

\subsection{Linear Dynamical System}\label{sec:lds}

In this experiment, we use a simple linear dynamical system with
quadratic costs over a finite horizon as our domain (similar to LQR)
for which we can compute the value function in closed form. The real
world has time-varying linear dynamics while the model class only has
time-invariant linear models making it agnostic. The cost function
penalizes control input at every timestep and the state only at the
final timestep (no intermediate state costs.) For detailed explanation
on the dynamics, cost functions, and model fitting losses, refer to
~\pref{app:lds}.~\pref{fig:small_mdp}(e)
shows the results for \ouralgmm{} and \sysid{}. \sysid{} converges
slowly to the expert performance trying to match the true dynamics at
every timestep while \ouralgmm{} using the value moment matching
objective quickly realizes that only the final state matters for cost
and finds a simple model which results in controls that bring the
state to zero by the end of the horizon. Thus, we see that \ouralgmm{}
by being value-aware can achieve significant statistical gains over
traditional MBRL methods.

\subsection{Mujoco Locomotion Benchmarks}\label{sec:gym}
\begin{figure*}
    \centering
    \includegraphics[width=0.243\linewidth]{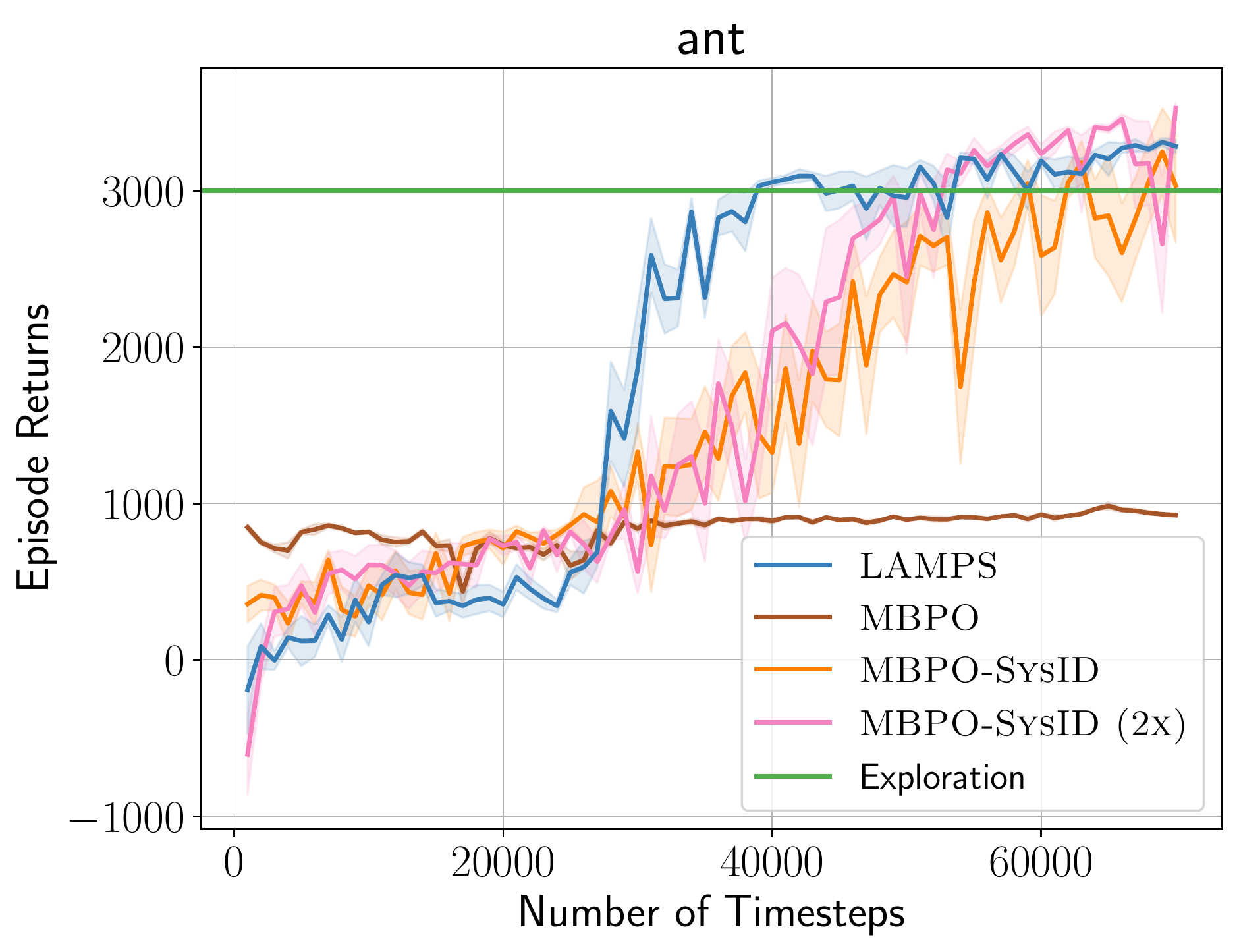}
    \includegraphics[width=0.243\linewidth]{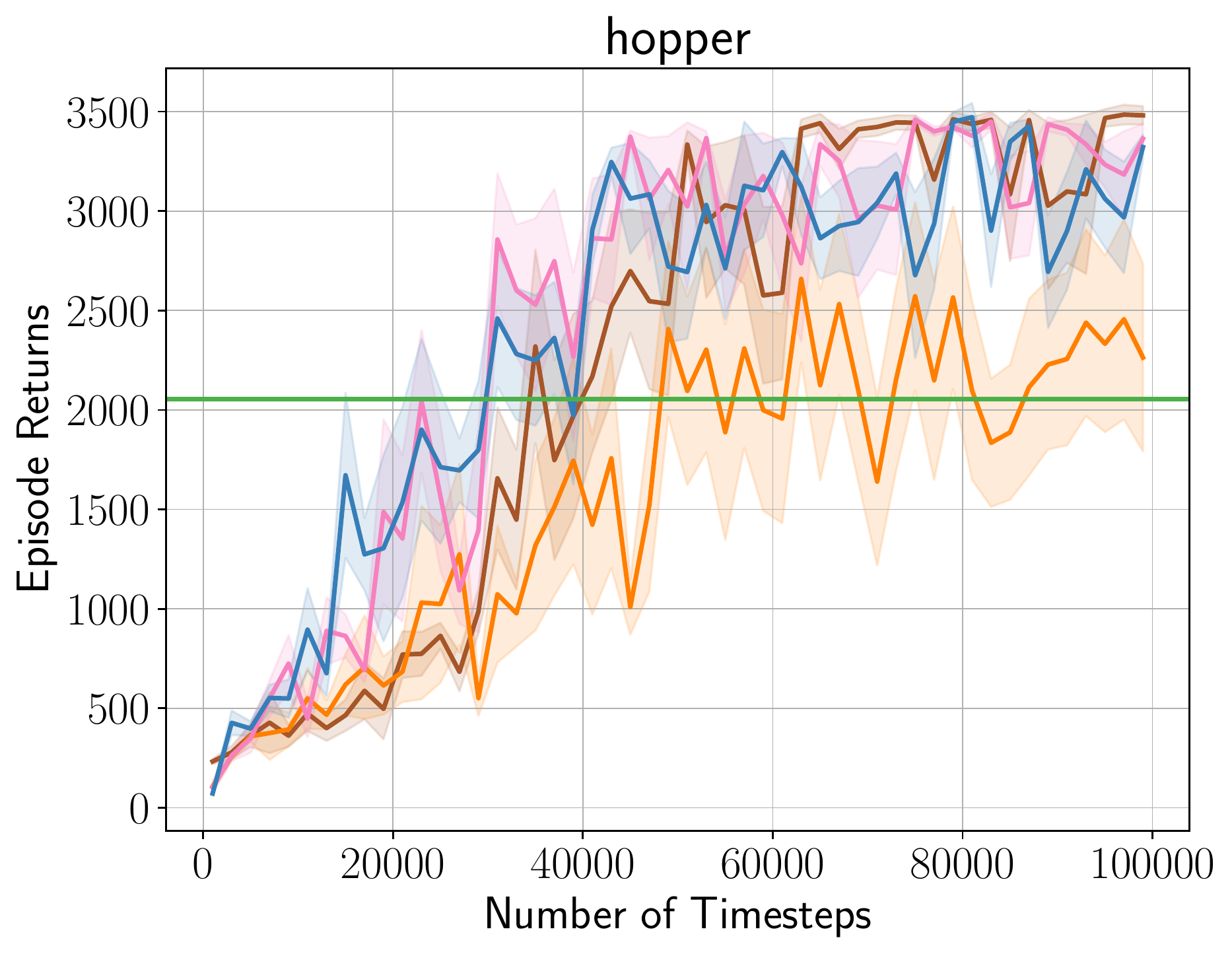}
    \includegraphics[width=0.243\linewidth]{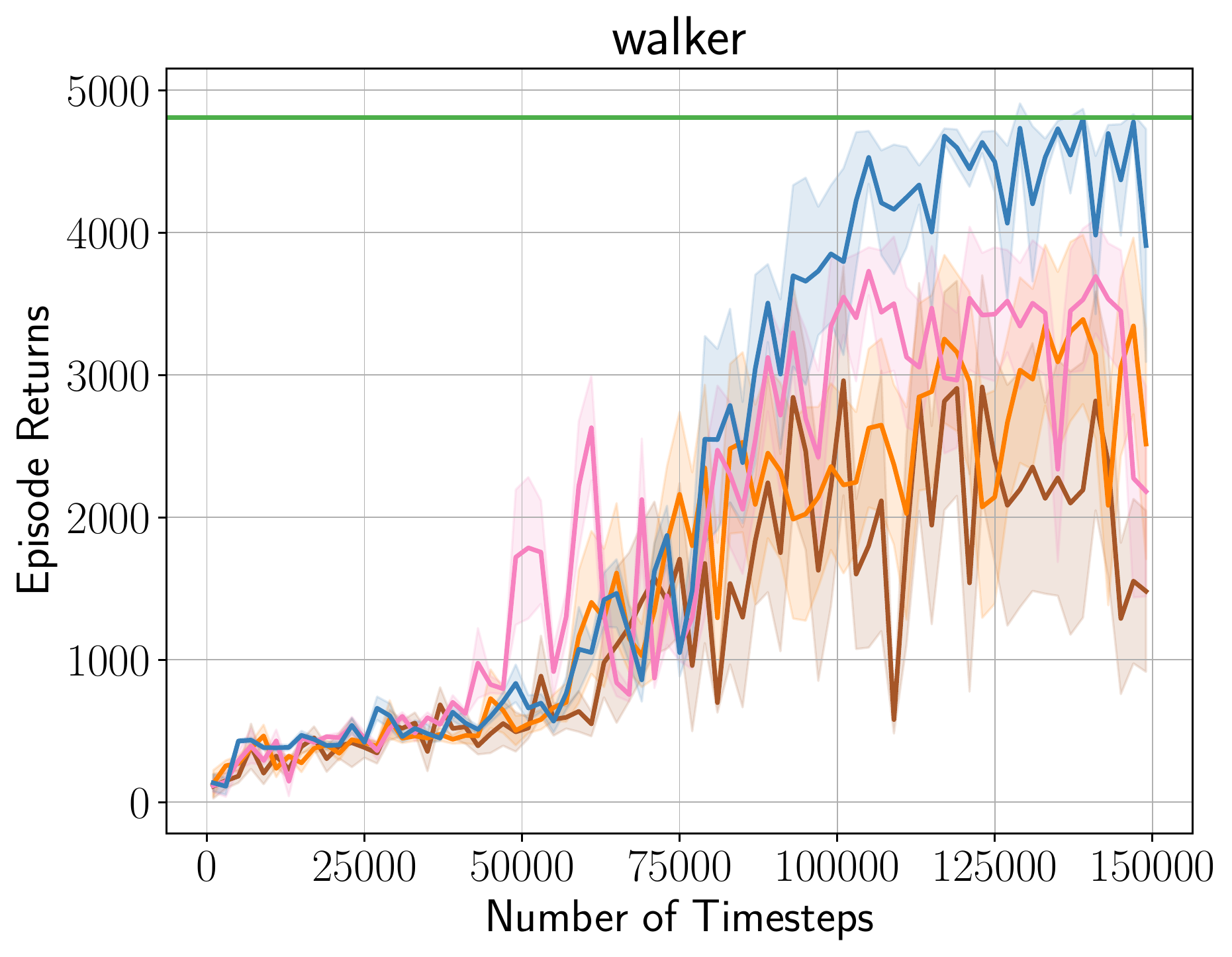}
    \includegraphics[width=0.243\linewidth]{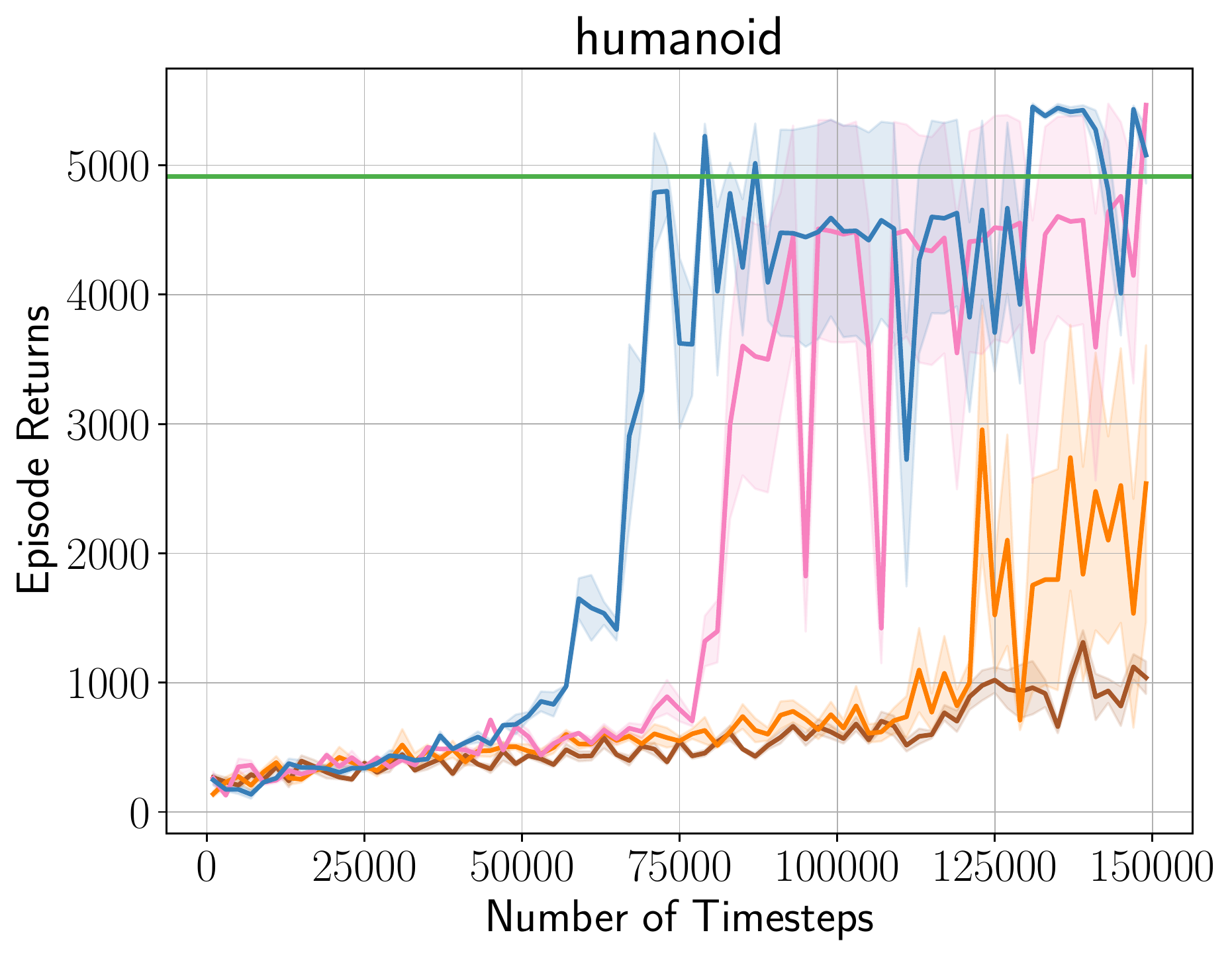}
    \caption{Results on mujoco locomotion benchmarks. All experiments
      are done using 5 random seeds and the shaded area denotes the
      standard error. We use 50000 exploration samples for humanoid
      task, and 10000 samples for all the other tasks. \mbpo{} also uses
      10000 random exploration samples as warm-start. The baseline
      ``Exploration'' denotes the average episodic returns of the
      exploration dataset.}
    \label{fig:gym}
\end{figure*}
In this experiment we test \ouralg{} on the standard dense reward mujoco
benchmarks~\cite{brockman2016openai}.  All baselines are implemented based on
\mbpo{}~\cite{janner2019trust}.
% For \sysid{}, we implement the
% \mbposysid{} baseline using both collected data and exploration data
% for model update. The branched update in \mbpo{} serves as an
% efficient surrogate for the optimal control.
In addition to \mbpo{}, we also design a variant \mbposysid{} which
also uses data from exploration distribution, similar to \sysid{}, for
model fitting. The branched update in \mbpo{} serves as an efficient
surrogate for optimal planning in \mbposysid{}.  Another variant
\mbposysidcompute{} doubles the number of policy updates and the
number of interactions with the learned model used when compared to
\mbposysid{}.  For \ouralg{}, we keep the model fitting procedure the
same as \mbposysid{}
and use states sampled from the exploration distribution for policy
updates, rather than the current policy's visitation distribution.
% To stick with our theory result, we use
% $\EE_{a \sim \pi} Q^{\hat \pi}_{\hat M}(s,a) - \EE_{a \sim
%   \pi^\ast(s)}Q^{\hat \pi}_{\hat M}(s,a)$ as the policy optimization
% objective on the state $s$ for $\pi$, where the second term serves as
% a good baseline for variance reduction.
For the exploration distribution $\nu$, we use an offline dataset and sample
from it every iteration. For more details on implementation such as
hyperparameters, refer to~\pref{app:mujoco}.

We show the results in \pref{fig:gym}. Compared to \mbpo{}, both
\ouralg{} and \mbposysid{} show better statistical efficiency, which
highlights the advantage of exploration
distribution~\cite{ross2012agnostic}. We note that \ouralg{}
consistently finds better policies with less number of
real world interactions than \mbposysid{} across all environments,
especially in humanoid which is the most difficult environment among
the ones used.
The performance of \mbposysidcompute{} shows that even when equipped
with twice the amount of computation as \ouralg{}, \ouralg{} still
outperforms or is competitive in all experiments. This highlights both
the computational and statistical efficiency of \ouralg{}. \looseness=-1
% On
% the computation side, note that in hopper and humanoid, SysId requires
% more computation to obtain similar performance to our algorithm, which
% also demonstrate the computation gain in practice. However, we indeed
% observe a failure case and we defer it to the \pref{app:failure} for
% more discussion.
%In our comparison with \citet{ross2012agnostic}, we investigate: a)
%Given the same number of computation (including number of
%interactions with the learned model and number of policy updates),
%what is the performance difference between the proposed method and
%SysId? b) How much more computation is required by SysId to achieve
%the same performance as the proposed algorithm.

\subsection{Maze}\label{sec:maze}

% In this final task, we investigate the implicit exploration ability of
% our proposed algorithm, comparing it with previous MBRL algorithms
% that also conduct no explicit exploration. We use the PointMaze
% environment \citep{fu2020d4rl} as the testbed, which originally serves
% as an offline RL benchmark. Here we only use a small subset of the
% offline dataset as our explore distribution, which only has partial
% coverage but still contains the expert traces from certain initial
% states.
Our final experiment investigates the performance of \ouralg{} in
sparse reward task by using PointMaze environment~\cite{fu2020d4rl} as
the domain. We use
only a small subset of the offline dataset as the exploration
distribution resulting in partial coverage and a small number of
expert trajectories. More details in~\pref{app:maze}. Since \ouralg{}
uses the exploration distribution in both model fitting and policy
computation steps, we expect it to outperform \mbposysid{}, which only
uses it in model fitting, as intelligent exploration is necessary in
sparse reward settings.~\pref{fig:maze} confirms our hypothesis where
\ouralg{} outperforms \mbposysid{} by a significant margin. By
focusing policy computation only along exploration distribution,
\ouralg{} does not exploit any inaccuracies elsewhere in the learned
model and quickly converges to a good policy.

% The intuition is, since \pref{alg:policy} optimizes the policy along
% the explore distribution, the policy should have implicitly explored
% the traces covered by the explore distribution, which contains the
% final goals. However traditional MBRL approaches still require
% exploring from scratch which is less efficient. We summarize our
% results in \pref{fig:maze}. We see in both setups, our algorithm
% stably outperforms SysID by a significant margin. In practice, we
% use a mixture of exploration distribution and learned policy
% visitation distribution because it better stabilizes the performance.
\begin{figure}
    \centering
    \includegraphics[width=0.45\linewidth]{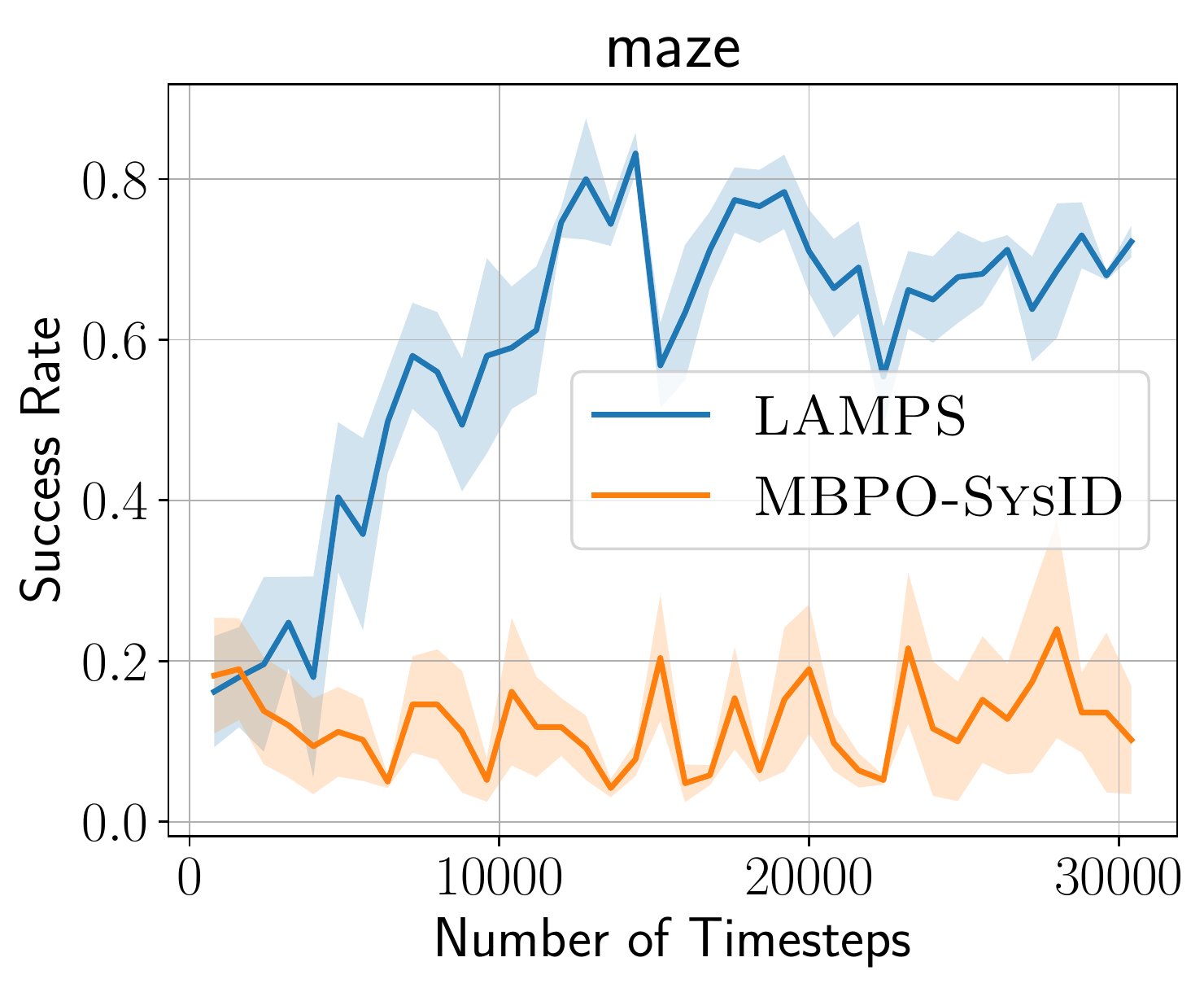}
    \includegraphics[width=0.45\linewidth]{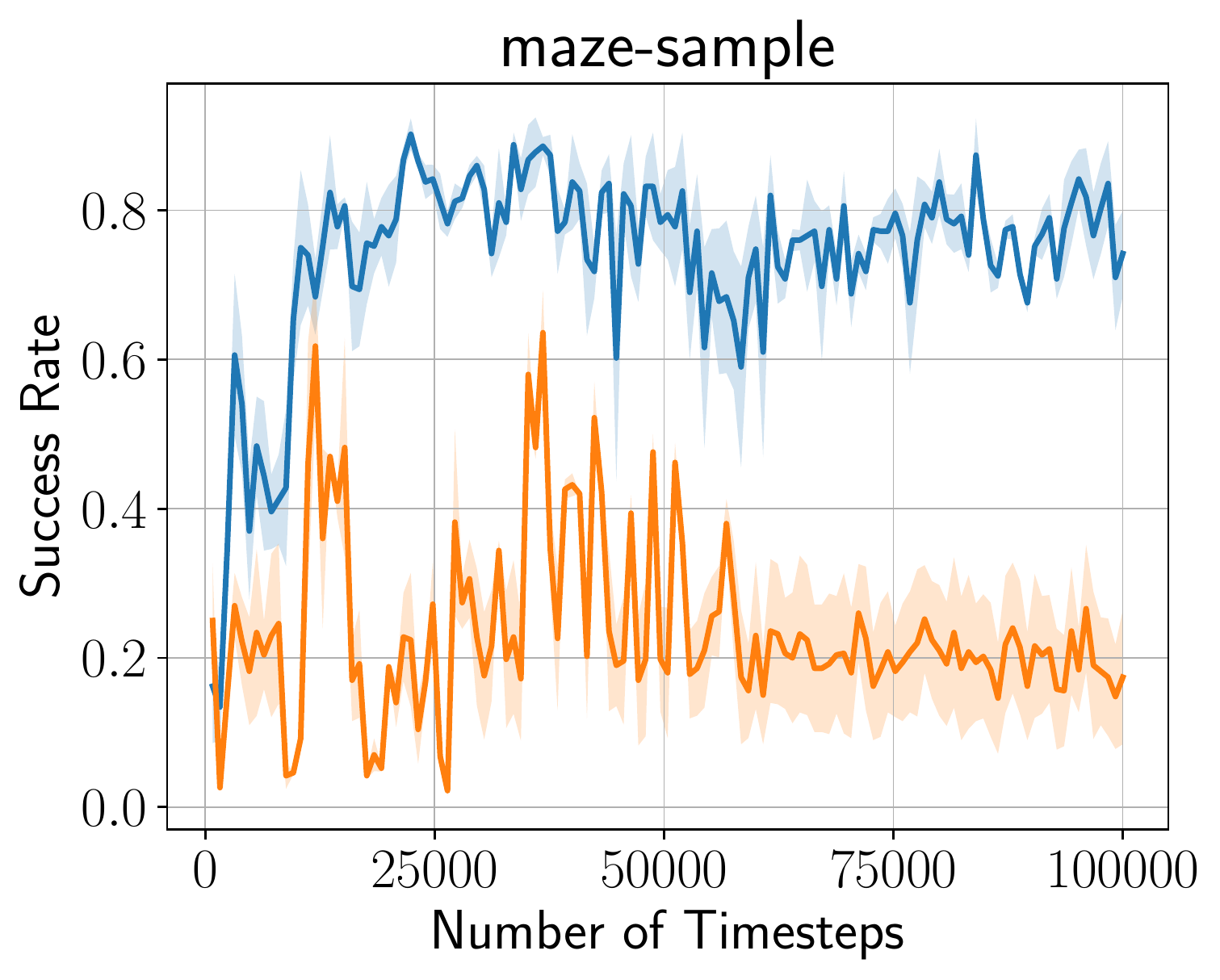}
    \caption{Results on D4RL PointMaze (large) environment. We use 10000
      exploration samples (left) and 50000 samples (right). In both
      setups, our algorithm shows a better exploration ability even
      though neither approach utilizes an explicit exploration
      scheme such as an exploration bonus.
      The results are averaged over 5 random seeds and the
      shaded area denotes the standard error.}
    \label{fig:maze}
\end{figure}

%%% Local Variables:
%%% mode: latex
%%% TeX-master: "../example_paper"
%%% End:

\section{Discussion}\label{sec:discussion}

In this work, we introduce a new unified objective function for
MBRL. The proposed objective function is designed to improve
computational efficiency in policy computation and alleviate the
objective mismatch issue in model fitting. Additionally, we present
two no-regret algorithms, \ouralg{} and \ouralgmm{}, that leverage the
proposed objective function and demonstrate their effectiveness
through statistical and computational gains on simulated
benchmarks. 

However, it should be noted that while \ouralg{} is
relatively straightforward to implement, \ouralgmm{} may be
challenging to apply to large MDPs where exact policy evaluation is
difficult. Additionally, both algorithms are sensitive to the quality
of exploration distribution $\nu$, and can only guarantee small regret
against policies with state-action distribution close to $\nu$. Hence,
we can expect these algorithms to compute a good policy if our prior
knowledge of the task allows us to design good exploration
distributions.

An interesting future work would be to extend this to the latent model setting, where we learn dynamics over an underlying latent state. In such a setting, the typical MLE model fitting objective does not intuitively make sense as we do not observe the underlying state and only have access to raw observations. We would also like to investigate if there exists a ``doubly-robust" version that combines the best of \sysid{} and \ouralg{} where we can take advantage of either having a good exploration distribution or a computationally cheap optimal planner.

% \jab{Two things to consider adding to discussion:
% 1)  Latent models where KL prediction error probably doesn't make sense.
% 2) Can we get a "doubly-robust" version of the algorithm by combining the two (SysID and LAmps) 
% where if *either* we have a good distribution or a good OC solver we do well. Seems hard
% because both require good prediction under the optimal policy. How do we get around that?}

% \jab{Also, didn't we want to cite the recent Csaba paper? What's the relationship between papers?}

%%% Local Variables:
%%% mode: latex
%%% TeX-master: "../example_paper"
%%% End:

\newpage

\bibliographystyle{icml2023}
\bibliography{ref}

\appendix
\onecolumn
\appendix
\section{Proofs}\label{app:proofs}
\subsection{Proofs for \pref{sec:framework}}
The
simulation lemma is useful to relate the performance of any policy
$\pi$, between two models, for example, the learned model $\hat M$ and
the real model $M^\ast$: %in the model $\Mhat$ with its performance in
                         %the real world $\Mstar$. %The lemma is more
                         %general as it can be used to compare the
                         %performance across any two models, but in
                         %our analysis we will focus on the real world
                         %and the learned model.

\begin{lemma}[Simulation Lemma]\label{lem:simulation}
    For any start distribution $\omega$, policy $\pi$, and transition functions $\Mhat$, $\Mstar$, we have
    \begin{align}
        &J_{\Mstar}^\omega(\pi) - J_{\Mhat}^\omega(\pi) \nonumber
        =\mathbb{E}_{s \sim \omega}[V^\pi_{\Mstar}(s) - V^\pi_{\Mhat}(s)]  \nonumber\\
        &=\frac{\gamma}{1 - \gamma}\mathbb{E}_{(s, a) \sim D_{\omega, \pi}}\left[ \mathbb{E}_{s' \sim \Mstar(s, a)}[V^\pi_{\Mhat}(s')] - \mathbb{E}_{s'' \sim \Mhat(s, a)}[V^\pi_{\Mhat}(s'')]\right]\label{eq:simulation-lemma}
    \end{align}
\end{lemma}
\begin{proof}
    The first equality follows from the definition of $J^\omega_{M}(\pi)$ as defined in \pref{sec:setup}. To prove the second equality we establish a recurrence as follows:
    \begin{align*}
        &\mathbb{E}_{s \sim \omega}[V^\pi_{\Mstar}(s) - V^\pi_{\Mhat}(s)]\\
        &=\expec_{s \sim \omega, a \sim \pi(s)}[c(s, a) + \gamma \expec_{s' \sim \Mstar(s, a)}[V^\pi_{\Mstar}(s')] - c(s, a) - \gamma \expec_{s'' \sim \Mhat(s, a)}[V^\pi_{\Mhat}(s'')]] \\
        &=\gamma\expec_{s \sim \omega, a \sim \pi(s)}[\expec_{s' \sim \Mstar(s, a)}[V^\pi_{\Mstar}(s')] - \expec_{s'' \sim \Mhat(s, a)}[V^\pi_{\Mhat}(s'')]] \\
        &= \gamma\expec_{s \sim \omega, a \sim \pi(s)}[\expec_{s' \sim \Mstar(s, a)}[V^\pi_{\Mstar}(s')] - \expec_{s' \sim \Mstar(s, a)}[V^\pi_{\Mhat}(s')]+\expec_{s' \sim \Mstar(s, a)}[V^\pi_{\Mhat}(s')] - \expec_{s'' \sim \Mhat(s, a)}[V^\pi_{\Mhat}(s'')]] \\
        &= \gamma \expec_{s' \sim d_{\omega, \pi}^1}[V^\pi_{\Mstar}(s') - V^\pi_{\Mhat}(s')] + \gamma\expec_{(s, a) \sim D^0_{\omega, \pi}}\left[ \mathbb{E}_{s' \sim \Mstar(s, a)}[V^\pi_{\Mhat}(s')] - \mathbb{E}_{s'' \sim \Mhat(s, a)}[V^\pi_{\Mhat}(s'')]\right]
    \end{align*}
    Thus, we established a recurrence between the performance difference at time $0$ and the performance difference at time $1$ with the state sampled from the state distribution by following $\pi$ at time $1$. We can solve this recurrence for the infinite horizon to get the lemma statement.
\end{proof}

The Performance difference via Planning in Model (PDPM) lemma is as follows:

\begin{lemma}[PDPM (\pref{lem:double} restate)]
    For any start state distribution $\omega$, policies $\hat{\pi}$, $\pi^\star$, and transition functions $\Mhat, \Mstar$ we have,
    \begin{align*}
        J_{\Mstar}^\omega(\pihat) - J_{\Mstar}^\omega(\pistar) = \mathbb{E}_{s \sim \omega} [V_{M^\star}^{\hat{\pi}}(s) - V_{M^\star}^{\pi^\star}(s)] =& \mathbb{E}_{s \sim \omega}[V^{\pihat}_{\Mhat}(s) - V^{\pistar}_{\Mhat}(s)]\\
        &+ \frac{\gamma}{1-\gamma}\mathbb{E}_{(s, a) \sim D_{\omega, \pihat}}[\expec_{s' \sim \Mstar(s, a)}[V_{\Mhat}^{\pihat}(s')] - \expec_{s'' \sim \Mhat(s, a)}[V_{\Mhat}^{\pihat}(s'')]] \\
        &+ \frac{\gamma}{1-\gamma}\mathbb{E}_{(s, a) \sim D_{\omega, \pistar}}[\expec_{s'' \sim \Mhat(s, a)}[V_{\Mhat}^{\pistar}(s'')] - \expec_{s' \sim \Mstar(s, a)}[V_{\Mhat}^{\pistar}(s')]]
    \end{align*}
\end{lemma}\begin{proof}
    We can add and subtract terms on the left hand side to get
    \begin{align*}
        &\mathbb{E}_{s \sim \omega} [V_{M^\star}^{\hat{\pi}}(s) - V_{M^\star}^{\pi^\star}(s)] = \\
        &\mathbb{E}_{s \sim \omega}\left[ (V^{\pihat}_{\Mhat}(s) - V_{\Mhat}^{\pistar}(s)) + (V^{\pihat}_{\Mstar}(s) - V^{\pihat}_{\Mhat}(s)) + (V^{\pistar}_{\Mhat}(s) - V^{\pistar}_{\Mstar}(s))\right]
    \end{align*}
    Apply the simulation lemma to the second and third terms inside the expectation above to get the result
    \begin{align*}
        \mathbb{E}_{s \sim \omega} [V_{M^\star}^{\hat{\pi}}(s) - V_{M^\star}^{\pi^\star}(s)] &= \mathbb{E}_{s \sim \omega}[V^{\pihat}_{\Mhat}(s) - V^{\pistar}_{\Mhat}(s)] \\
        &+ \frac{\gamma}{1-\gamma}\mathbb{E}_{(s, a) \sim d_{\omega, \pihat}, s' \sim \Mstar(s, a), s'' \sim \Mhat(s, a)}[V_{\Mhat}^{\pihat}(s') - V_{\Mhat}^{\pihat}(s'')] \\
        &+ \frac{\gamma}{1-\gamma}\mathbb{E}_{(s, a) \sim d_{\omega, \pistar}, s' \sim \Mstar(s, a), s'' \sim \Mhat(s, a)}[V_{\Mhat}^{\pistar}(s'') - V_{\Mhat}^{\pistar}(s')]
    \end{align*}
\end{proof}

\begin{corollary}
[Corollary~\pref{cor:tv} restate]
    For any start state distribution $\omega$, $\pi^\star$, and transition functions $\Mhat, \Mstar$, let $\hat \pi$ be the returned optimal control policy in $\hat M$ as in \pref{eq:oc}, we have,
    \begin{align*}
        \mathbb{E}_{s \sim \omega} &[V_{M^\star}^{\hat{\pi}}(s) -  V_{M^\star}^{\pi^\star}(s)] \leq \epsilon_{oc} + \\ &\frac{\gamma \hat V_{\max}}{1-\gamma}\mathbb{E}_{(s, a) \sim D_{\omega, \pihat}}\left\|\Mhat(s, a)-\Mstar(s, a)\right\|_1
        + \\&\frac{\gamma V_{\max}}{1-\gamma}\mathbb{E}_{(s, a) \sim D_{\omega, \pistar}}\left\|\Mhat(s, a)-\Mstar(s, a)\right\|_1,
    \end{align*}
where $\hat V_{\max} = \|V_{\hat M}^{\hat \pi}\|_{\infty}, V_{\max} = \|V^{\pi^\ast}_{\hat M}\|_{\infty}$.
\end{corollary}
\begin{proof}
    By \pref{lem:double}, we have:
    \begin{align*}
        \mathbb{E}_{s \sim \omega} [V_{M^\star}^{\hat{\pi}}(s) - V_{M^\star}^{\pi^\star}(s)] &= \mathbb{E}_{s \sim \omega}[V^{\pihat}_{\Mhat}(s) - V^{\pistar}_{\Mhat}(s)] \\
        &+ \frac{\gamma}{1-\gamma}\mathbb{E}_{(s, a) \sim d_{\omega, \pihat}, s' \sim \Mstar(s, a), s'' \sim \Mhat(s, a)}[V_{\Mhat}^{\pihat}(s') - V_{\Mhat}^{\pihat}(s'')] \\
        &+ \frac{\gamma}{1-\gamma}\mathbb{E}_{(s, a) \sim d_{\omega, \pistar}, s' \sim \Mstar(s, a), s'' \sim \Mhat(s, a)}[V_{\Mhat}^{\pistar}(s'') - V_{\Mhat}^{\pistar}(s')],
    \end{align*}
Then we bound the first term by \pref{eq:oc}, and by holder's inequality, the second term is bounded by
\begin{align*}
    \frac{\gamma}{1-\gamma}\mathbb{E}_{(s, a) \sim d_{\omega, \pihat}, s' \sim \Mstar(s, a), s'' \sim \Mhat(s, a)}[V_{\Mhat}^{\pihat}(s') - V_{\Mhat}^{\pihat}(s'')] \leq \frac{\gamma}{1-\gamma} \|V_{\Mhat}^{\pihat}\|_{\infty} \EE_{s,a\sim D_{\omega, \pihat}}\left\|\Mhat(s, a)-\Mstar(s, a)\right\|_1,
\end{align*}
and apply holder's inequality to the third term similarly, we complete the proof.
\end{proof}

\subsection{Proofs for \pref{sec:simulation}}\label{sec:appendix-simulation}
In this section, we present the proof for  \pref{sec:simulation}. Let us start with the Performance Difference via Advantage in Model (PDAM) Lemma:
\begin{lemma}[PDAM (restate of \pref{lem:simulation-pdl})]
    Given any start state distribution $\omega$, policies $\pihat, \pistar$, and transition functions $\Mhat, \Mstar$ we have:
     \begin{align*}
        &J_{\Mstar}^\omega(\pihat) - J_{\Mstar}^\omega(\pistar) = \mathbb{E}_{s \sim \omega} [V_{M^\star}^{\hat{\pi}}(s) - V_{M^\star}^{\pi^\star}(s)] =\nonumber\\
        &\frac{\gamma}{1 - \gamma}\mathbb{E}_{(s, a) \sim D_{\omega, \pihat}}[\expec_{s' \sim \Mstar(s, a)}[V_{\Mhat}^{\pihat}(s')] - \expec_{s'' \sim \Mhat(s, a)}[V_{\Mhat}^{\pihat}(s'')]] + \nonumber\\
        &\frac{\gamma}{1 - \gamma}\expec_{(s, a) \sim D_{\omega, \pistar}}[\expec_{s'' \sim \Mhat(s, a)}[V^{\pihat}_{\Mhat}(s'')] - \expec_{s' \sim \Mstar(s, a)}[V^{\pihat}_{\Mhat}(s')]] + \nonumber \\
        &\frac{1}{1 - \gamma}\expec_{s \sim d_{\omega, \pistar}}[ V^{\pihat}_{\Mhat}(s) - \expec_{a \sim \pistar(s)}[Q^{\pihat}_{\Mhat}(s, a)]]
    \end{align*}
\end{lemma}

\begin{proof}
    Let's begin with the left hand side, and reformulate it as follows:
    \begin{align*}
        \mathbb{E}_{s \sim \omega}[V_{\Mstar}^{\pihat}(s) - V^{\pistar}_{\Mstar}(s)]
        %&=\mathbb{E}_{s \sim \omega}\left[ (V^{\pihat}_{\Mhat}(s) - V_{\Mstar}^{\pistar}(s)) + (V^{\pihat}_{\Mstar}(s) - V^{\pihat}_{\Mhat}(s))\right] \\
        = \mathbb{E}_{s \sim \omega}[V^{\pihat}_{\Mhat}(s) - V_{\Mstar}^{\pistar}(s)] + \mathbb{E}_{s \sim \omega}[V^{\pihat}_{\Mstar}(s) - V^{\pihat}_{\Mhat}(s)]
    \end{align*}
    The second term above is familiar to us, it is the left hand side of the simulation lemma in equation~\eqref{eq:simulation-lemma}. So we can apply the simulation lemma to get:
    \begin{align}
        \mathbb{E}_{s \sim \omega}[V_{\Mstar}^{\pihat}(s) - V^{\pistar}_{\Mstar}(s)] = \mathbb{E}_{s \sim \omega}[V^{\pihat}_{\Mhat}(s) - V_{\Mstar}^{\pistar}(s)] +\frac{\gamma}{1 - \gamma}\mathbb{E}_{(s, a) \sim D_{\omega, \pihat}}\left[ \mathbb{E}_{s' \sim \Mstar(s, a)}[V^{\pihat}_{\Mhat}(s')] - \mathbb{E}_{s'' \sim \Mhat(s, a)}[V^{\pihat}_{\Mhat}(s'')]\right]\label{eq:first_term}
    \end{align}

    Now all that remains is the first term which can be simplified as:
    \begin{align*}
        &\expec_{s\sim \omega} [V^{\pihat}_{\Mhat}(s) - V^{\pistar}_{\Mstar}(s)] \\
        &=\expec_{s\sim\omega}[V^{\pihat}_{\Mhat}(s) - \expec_{a \sim \pistar(s)}Q^{\pihat}_{\Mhat}(s, a) + \expec_{a\sim\pistar(s)}Q^{\pihat}_{\Mhat}(s, a) - V^{\pistar}_{\Mstar}(s)] \\
        &=\expec_{s\sim\omega}[V^{\pihat}_{\Mhat}(s) - \expec_{a \sim \pistar(s)}Q^{\pihat}_{\Mhat}(s, a)] \\
        &~~~~+\expec_{s\sim\omega}\left[\expec_{a\sim\pistar(s)}[c(s, a) + \gamma\expec_{s'' \sim \Mhat(s, a)}V^{\pihat}_{\Mhat}(s'')] \right.\\
        &~~~~~~~~- \left.\expec_{a \sim\pistar(s)}[c(s, a) + \gamma\expec_{s' \sim \Mstar(s, a)}V^{\pistar}_{\Mstar}(s')] \right] \\
        &= \expec_{s\sim\omega}[V^{\pihat}_{\Mhat}(s) - \expec_{a \sim \pistar(s)}Q^{\pihat}_{\Mhat}(s, a)] \\
        &~~~~+\gamma\expec_{(s, a) \sim D_{\omega, \pistar}^0}[\expec_{s'' \sim \Mhat(s, a)}V^{\pihat}_{\Mhat}(s'') - \expec_{s' \sim \Mstar(s, a)}V^{\pistar}_{\Mstar}(s')] \\
        &= \expec_{s\sim\omega}[V^{\pihat}_{\Mhat}(s) - \expec_{a \sim \pistar(s)}Q^{\pihat}_{\Mhat}(s, a)] \\
        &~~~~+\gamma\expec_{(s, a) \sim D_{\omega, \pistar}^0}\left[ \expec_{s'' \sim \Mhat(s, a)}V^{\pihat}_{\Mhat}(s'') - \expec_{s' \sim \Mstar(s, a)}V^{\pihat}_{\Mhat}(s')\right.\\
        &~~~~~~~~+\left. \expec_{s' \sim \Mstar(s, a)}V^{\pihat}_{\Mhat}(s') - \expec_{s' \sim \Mstar(s, a)}V^{\pistar}_{\Mstar}(s') \right] \\
        &= \expec_{s\sim\omega}[V^{\pihat}_{\Mhat}(s) - \expec_{a \sim \pistar(s)}Q^{\pihat}_{\Mhat}(s, a)] \\
        &~~~~+\gamma\expec_{(s, a) \sim D_{\omega, \pistar}^0}[ \expec_{s'' \sim \Mhat(s, a)}V^{\pihat}_{\Mhat}(s'') - \expec_{s' \sim \Mstar(s, a)}V^{\pihat}_{\Mhat}(s')] \\
        &~~~~+\gamma\expec_{(s, a) \sim D_{\omega, \pistar}^0}[\expec_{s' \sim \Mstar(s, a)}V^{\pihat}_{\Mhat}(s') - \expec_{s' \sim \Mstar(s, a)}V^{\pistar}_{\Mstar}(s')] \\
        &= \expec_{s\sim\omega}[V^{\pihat}_{\Mhat}(s) - \expec_{a \sim \pistar(s)}Q^{\pihat}_{\Mhat}(s, a)] \\
        &~~~~+\gamma\expec_{(s, a) \sim D_{\omega, \pistar}^0}[ \expec_{s'' \sim \Mhat(s, a)}V^{\pihat}_{\Mhat}(s'') - \expec_{s' \sim \Mstar(s, a)}V^{\pihat}_{\Mhat}(s')] \\
        &~~~~+\gamma\expec_{s' \sim d_{\omega, \pistar}^1}[V^{\pihat}_{\Mhat}(s') - V^{\pistar}_{\Mstar}(s')]
    \end{align*}

    Solving the above recurrence to the infinite horizon we obtain:
    \begin{align*}
        &\mathbb{E}_{s \sim \omega}[V^{\pihat}_{\Mhat}(s) - V_{\Mstar}^{\pistar}(s)] = \\ &~~~~~\frac{\gamma}{1-\gamma}\expec_{(s, a) \sim D_{\omega, \pistar}}[\expec_{s'' \sim \Mhat(s, a)}V^{\pihat}_{\Mhat}(s'') - \expec_{s' \sim \Mstar(s, a)}V^{\pihat}_{\Mhat}(s')] \\
        &~~~~+\frac{1}{1-\gamma}\expec_{s \sim d_{\omega, \pistar}}[V^{\pihat}_{\Mhat}(s) - \expec_{a \sim \pistar(s)}[Q^{\pihat}_{\Mhat}(s, a)]]
    \end{align*}
    By combining this with our previous result using Simulation Lemma in~\eqref{eq:first_term}, we can complete the proof.
    % \begin{align*}
    %     &\mathbb{E}_{s \sim \omega} [V_{M^\star}^{\hat{\pi}}(s) - V_{M^\star}^{\pi^\star}(s)] \\
    %     &~~~~~= \frac{\gamma}{1 - \gamma}\mathbb{E}_{(s, a) \sim D_{\omega, \pihat}}[\expec_{s' \sim \Mstar(s, a)}[V_{\Mhat}^{\pihat}(s')] - \expec_{s'' \sim \Mhat(s, a)}[V_{\Mhat}^{\pihat}(s'')]] \\
    %     &~~~~~~~~~ + \frac{\gamma}{1-\gamma}\expec_{(s, a) \sim D_{\omega, \pistar}}[\expec_{s'' \sim \Mhat(s, a)}[V^{\pihat}_{\Mhat}(s'')] - \expec_{s' \sim \Mstar(s, a)}[V^{\pihat}_{\Mhat}(s')]] \\
    %     &~~~~~~~~~ + \frac{1}{1-\gamma}\expec_{s \sim d_{\omega, \pistar}}[ V^{\pihat}_{\Mhat}(s) - \expec_{a \sim \pistar(s)}[Q^{\pihat}_{\Mhat}(s, a)]]
    % \end{align*}
\end{proof}

Now, we show the results using the exploration distribution $\nu$ and coverage coefficient $\Ccal$:
\begin{corollary}  \label{cor:simulation}
    Let $\nu$ be the exploration distribution, and let $\Ccal$ be the coverage coefficient. Given any start state distribution $\omega$, policies $\pihat, \pistar$, and transition functions $\Mhat, \Mstar$ we have:
     \begin{align*}
        &J_{\Mstar}^\omega(\pihat) - J_{\Mstar}^\omega(\pistar) = \mathbb{E}_{s \sim \omega} [V_{M^\star}^{\hat{\pi}}(s) - V_{M^\star}^{\pi^\star}(s)] \leq\nonumber\\
        &\frac{\gamma}{1 - \gamma}\mathbb{E}_{(s, a) \sim D_{\omega, \pihat}}[\expec_{s' \sim \Mstar(s, a)}[V_{\Mhat}^{\pihat}(s')] - \expec_{s'' \sim \Mhat(s, a)}[V_{\Mhat}^{\pihat}(s'')]] + \nonumber\\
        &\frac{\gamma \Ccal}{1 - \gamma}\expec_{(s, a) \sim \nu}[\expec_{s'' \sim \Mhat(s, a)}[V^{\pihat}_{\Mhat}(s'')] - \expec_{s' \sim \Mstar(s, a)}[V^{\pihat}_{\Mhat}(s')]] + \nonumber \\
        &\frac{\Ccal}{1 - \gamma}\expec_{s \sim \nu}[ V^{\pihat}_{\Mhat}(s) - \expec_{a \sim \nu(\cdot \mid s)}[Q^{\pihat}_{\Mhat}(s, a)]]
    \end{align*}
\end{corollary}
\begin{proof}
    \pref{lem:simulation-pdl} gives us:
    \begin{align*}
        &J_{\Mstar}^\omega(\pihat) - J_{\Mstar}^\omega(\pistar) = \mathbb{E}_{s \sim \omega} [V_{M^\star}^{\hat{\pi}}(s) - V_{M^\star}^{\pi^\star}(s)] \\
        =&\frac{\gamma}{1 - \gamma}\mathbb{E}_{(s, a) \sim D_{\omega, \pihat}}[\expec_{s' \sim \Mstar(s, a)}[V_{\Mhat}^{\pihat}(s')] - \expec_{s'' \sim \Mhat(s, a)}[V_{\Mhat}^{\pihat}(s'')]] \\
        +&\frac{\gamma}{1 - \gamma}\expec_{(s, a) \sim D_{\omega, \pistar}}[\expec_{s'' \sim \Mhat(s, a)}[V^{\pihat}_{\Mhat}(s'')] - \expec_{s' \sim \Mstar(s, a)}[V^{\pihat}_{\Mhat}(s')]] \\
        +&\frac{1}{1 - \gamma}\expec_{(s,a) \sim D_{\omega, \pistar}}[ V^{\pihat}_{\Mhat}(s) - Q^{\pihat}_{\Mhat}(s, a)],
    \end{align*}
    Then let $\nu$ be the explore distribution, we have:
    \begin{align*}
        &J_{\Mstar}^\omega(\pihat) - J_{\Mstar}^\omega(\pistar) = \mathbb{E}_{s \sim \omega} [V_{M^\star}^{\hat{\pi}}(s) - V_{M^\star}^{\pi^\star}(s)] \\
        \leq &\frac{\gamma}{1 - \gamma}\mathbb{E}_{(s, a) \sim D_{\omega, \pihat}}[\expec_{s' \sim \Mstar(s, a)}[V_{\Mhat}^{\pihat}(s')] - \expec_{s'' \sim \Mhat(s, a)}[V_{\Mhat}^{\pihat}(s'')]] \\
        +&\Ccal\frac{\gamma}{1 - \gamma}\expec_{(s, a) \sim D_{e}}[\expec_{s'' \sim \Mhat(s, a)}[V^{\pihat}_{\Mhat}(s'')] - \expec_{s' \sim \Mstar(s, a)}[V^{\pihat}_{\Mhat}(s')]] \\
        +&\Ccal\frac{1}{1 - \gamma}\expec_{(s,a) \sim D_{e}}[ V^{\pihat}_{\Mhat}(s) - Q^{\pihat}_{\Mhat}(s, a)],
    \end{align*}
    where the first term is by $\Ccal \geq 1$, and the last two are by
    importance sampling.
\end{proof}

\subsection{Proof for \pref{sec:alg_mle}}
The following result will be stated in terms of expert distribution
$D_{\omega,\pi^\ast}$ for simplicity. We show in
Corollary~\pref{cor:regret_cov} that this can be extended to the case
when we only have access to an exploration distribution $\nu$.

\begin{corollary}
    [Corollary~\pref{cor:tv_new} restate]
        For any start state distribution $\omega$, $\pi^\star$, and transition functions $\Mhat, \Mstar$, we have,
\begin{align*}
    J_{\Mstar}^\omega(\pihat) &- J_{\Mstar}^\omega(\pistar) = \mathbb{E}_{s \sim \omega} [V_{M^\star}^{\hat{\pi}}(s) - V_{M^\star}^{\pi^\star}(s)] \nonumber\\
        &\leq \frac{\gamma \hat V_{\max}}{1 - \gamma}\mathbb{E}_{(s, a) \sim D_{\omega, \pihat}}||\Mhat(s, a) - \Mstar(s, a)||_1 \nonumber\\
        &~~ + \frac{\gamma \hat V_{\max}}{1-\gamma}\expec_{(s, a) \sim D_{\omega, \pistar}}||\Mhat(s, a) - \Mstar(s, a)||_1 \nonumber \\
        &~~ + \frac{1}{1-\gamma}\expec_{s \sim d_{\omega, \pistar}}[ V^{\pihat}_{\Mhat}(s) - \expec_{a \sim \pistar(s)}[Q^{\pihat}_{\Mhat}(s, a)]]
\end{align*}
\end{corollary}
\begin{proof}
    By \pref{lem:simulation-pdl}, we have:
     \begin{align*}
        &J_{\Mstar}^\omega(\pihat) - J_{\Mstar}^\omega(\pistar) = \mathbb{E}_{s \sim \omega} [V_{M^\star}^{\hat{\pi}}(s) - V_{M^\star}^{\pi^\star}(s)] =\nonumber\\
        &\frac{\gamma}{1 - \gamma}\mathbb{E}_{(s, a) \sim D_{\omega, \pihat}}[\expec_{s' \sim \Mstar(s, a)}[V_{\Mhat}^{\pihat}(s')] - \expec_{s'' \sim \Mhat(s, a)}[V_{\Mhat}^{\pihat}(s'')]] + \nonumber\\
        &\frac{\gamma}{1 - \gamma}\expec_{(s, a) \sim D_{\omega, \pistar}}[\expec_{s'' \sim \Mhat(s, a)}[V^{\pihat}_{\Mhat}(s'')] - \expec_{s' \sim \Mstar(s, a)}[V^{\pihat}_{\Mhat}(s')]] + \nonumber \\
        &\frac{1}{1 - \gamma}\expec_{s \sim d_{\omega, \pistar}}[ V^{\pihat}_{\Mhat}(s) - \expec_{a \sim \pistar(s)}[Q^{\pihat}_{\Mhat}(s, a)]].
    \end{align*}
    Applying holder's inequality to the first two terms completes the proof.
\end{proof}

\begin{theorem}[\pref{thm:policy} restate]
Let $\{\hat \pi_t\}^T_{t=1}$ be the sequence of returned policies of
  \ouralg, we have:
    \begin{align*}
      \frac{1}{T}\sum_{t=1}^T J_{\Mstar}^\omega(\pihat_t) - J_{\Mstar}^\omega(\pistar) \leq
                              \tilde O\left( \epsilon_{po} + \frac{ \hat
                                V_{\max}}{1-\gamma}\left(\sqrt{\epsilon^{KL}_{model}} +
                                \frac{1}{\sqrt{T}}\right)\right),
    \end{align*}
    where $\hat V_{\max} = \|V^{\hat \pi}_{\hat M}\|_{\infty}$,
    and
    $\epsilon^{KL}_{model} = \min_{M \in \Mcal}\expec_{s,a \sim \bar
      \Dcal_T} \mathsf{KL}(M(s,a), M^{\ast}(s,a))$ is the agnostic
    model error.\end{theorem}
\begin{proof}
    Similar to \citet{ross2012agnostic}, this proof is to establish the model error guarantee from running \pref{alg:model_mle}. First, by Corollary~\pref{cor:tv_new}, we have
    \begin{align*}
        &\sum_{t=1}^T J_{\Mstar}^\omega(\pihat_t) - J_{\Mstar}^\omega(\pistar) \\
        \leq& \sum_{t=1}^T  \frac{\gamma \hat V_{\max}}{1 - \gamma}\mathbb{E}_{(s, a) \sim D_{\omega, \pihat_t}}||\Mhat_t(s, a) - \Mstar(s, a)||_1 \nonumber +  \frac{\gamma \hat V_{\max}}{1-\gamma}\expec_{(s, a) \sim D_{\omega, \pistar}}||\Mhat_t(s, a) - \Mstar(s, a)||_1 \nonumber \\
        &~~ + \sum_{t=1}^T \frac{1}{1-\gamma}\expec_{s \sim d_{\omega, \pistar}}[ V^{\pihat_t}_{\Mhat_t}(s) - \expec_{a \sim \pistar(s)}[Q^{\pihat_t}_{\Mhat_t}(s, a)]]\\
         \leq& \sum_{t=1}^T  \frac{\gamma \hat V_{\max}}{1 - \gamma}\mathbb{E}_{(s, a) \sim D_{\omega, \pihat_t}}||\Mhat_t(s, a) - \Mstar(s, a)||_1 \nonumber +  \frac{\gamma \hat V_{\max}}{1-\gamma}\expec_{(s, a) \sim D_{\omega, \pistar}}||\Mhat_t(s, a) - \Mstar(s, a)||_1 \nonumber + T\epsilon_{po},
    \end{align*}
where the last line is by running \pref{alg:policy}. To bound the model error, recall the MLE model loss function:
\begin{align*}
    \ell_t(M) = \EE_{s,a,s' \sim \Dcal_{t}} \log M(s'\mid s,a),
\end{align*}
then running FTL as in \pref{alg:model_mle} for $T$ rounds gives us:
\begin{align*}
    \sum_{t=1}^T \ell_t(\hat M_{t}) &\leq \min_{M \in \Mcal} \sum_{t=1}^T \ell_t(M) + O(\log(T)) \\
    \sum_{t=1}^T \ell_t(\hat M_{t}) + 2\EE_{s,a \sim \Dcal_t}\EE_{s' \sim M^{\ast}(s,a)} \log(M^\ast(s' \mid s,a)) &\leq  \min_{M \in \Mcal} \sum_{t=1}^T \ell_t(M) + 2\EE_{s,a \sim \Dcal_t}\EE{s' \sim M^{\ast}(s,a)} \log(M^\ast(s' \mid s,a))  + O(\log(T)) \\
    \sum_{t=1}^T 2\EE_{s,a \sim \Dcal_t} \mathsf{KL}(\hat M_t(s,a), M^\ast(s,a)) &\leq \min_{M \in \Mcal} \sum_{t=1}^T 2\EE_{s,a \sim \Dcal_t} \mathsf{KL}( M_t(s,a), M^\ast(s,a)) + O(\log(T)) \\
    2\sum_{t=1}^T \EE_{s,a \sim \Dcal_t} \mathsf{KL}(\hat M_t(s,a), M^\ast(s,a)) &\leq 2T \epsilon^{KL}_{model} + O(\log(T)),
\end{align*}
Recall again $\Dcal_t = \frac{1}{2}D_{\omega, \pi_t} + \frac{1}{2}D_{\omega,\pi^\ast}$. Then by Pinsker's inequality and Jensen's inequality, we have:
\begin{align*}
    \sum_{t=1}^T \EE_{s,a \sim \Dcal_t} \|\hat M^t(s,a) - M^\ast(s,a)\|_1 &\leq \sum_{t=1}^T \sqrt{2 \EE_{s,a \sim \Dcal_t} \mathsf{KL}(\hat M_t(s,a), M^\ast(s,a))} \\
    &\leq T\sqrt{\frac{1}{T}\sum_{t=1}^T 2 \EE_{s,a \sim \Dcal_t} \mathsf{KL}(\hat M_t(s,a), M^\ast(s,a))} \\
    &\leq 2T \sqrt{\epsilon^{KL}_{model}} + \tilde O(\sqrt{T}).
\end{align*}
Thus we have
\begin{align*}
    &\sum_{t=1}^T  \frac{\gamma \hat V_{\max}}{1 - \gamma}\mathbb{E}_{(s, a) \sim D_{\omega, \pihat^t}}||\Mhat_t(s, a) - \Mstar(s, a)||_1  +  \frac{\gamma \hat V_{\max}}{1-\gamma}\expec_{(s, a) \sim D_{\omega, \pistar}}||\Mhat_t(s, a) - \Mstar(s, a)||_1 \\
    \leq &\frac{\gamma \hat V_{\max}}{1 - \gamma}\left\{2T \sqrt{\epsilon^{KL}_{model}} + \tilde O(\sqrt{T})\right\},
\end{align*}
and finally multiply both side by $\frac{1}{T}$ and we complete the proof.
\end{proof}

Finally, we show that the results easily extend to the exploration distribution setup.

\begin{corollary}\label{cor:regret_cov}
Let $\{\hat \pi_t\}^T_{t=1}$ be the sequence of returned policies of
  \ouralg, we have:
    \begin{align*}
      \frac{1}{T}\sum_{t=1}^T J_{\Mstar}^\omega(\pihat_t) - J_{\Mstar}^\omega(\pistar) \leq
                              \tilde O\left(\Ccal \epsilon_{po} + \frac{\Ccal \hat
                                V_{\max}}{1-\gamma}\left(\sqrt{\epsilon^{KL}_{model}} +
                                \frac{1}{\sqrt{T}}\right)\right),
    \end{align*}
    where $\hat V_{\max} = \|V^{\hat \pi}_{\hat M}\|_{\infty}$,
    $\Ccal = \sup_{s,a} \frac{D_{\omega, \pi^\ast}(s,a)}{\nu(s,a)}$,
    and
    $\epsilon^{KL}_{model} = \min_{M \in \Mcal}\expec_{s,a \sim \bar
      \Dcal_T} \mathsf{KL}(M(s,a), M^{\ast}(s,a))$ is the agnostic
    model error.
\end{corollary}

\begin{proof}
    We start from Corollary~\pref{cor:simulation}, which gives us:
    \begin{align*}
        &J_{\Mstar}^\omega(\pihat) - J_{\Mstar}^\omega(\pistar) = \mathbb{E}_{s \sim \omega} [V_{M^\star}^{\hat{\pi}}(s) - V_{M^\star}^{\pi^\star}(s)] \\
        \leq &\Ccal\frac{\gamma}{1 - \gamma}\mathbb{E}_{(s, a) \sim D_{\omega, \pihat}}[\expec_{s' \sim \Mstar(s, a)}[V_{\Mhat}^{\pihat}(s')] - \expec_{s'' \sim \Mhat(s, a)}[V_{\Mhat}^{\pihat}(s'')]] \\
        +&\Ccal\frac{\gamma}{1 - \gamma}\expec_{(s, a) \sim D_{e}}[\expec_{s'' \sim \Mhat(s, a)}[V^{\pihat}_{\Mhat}(s'')] - \expec_{s' \sim \Mstar(s, a)}[V^{\pihat}_{\Mhat}(s')]] \\
        +&\Ccal\frac{1}{1 - \gamma}\expec_{(s,a) \sim D_{e}}[ V^{\pihat}_{\Mhat}(s) - Q^{\pihat}_{\Mhat}(s, a)],
    \end{align*}
    where the first term is by $\Ccal \geq 1$, and the last two are by importance ratio. Then let
    \begin{align*}
        \ell_t(M) = \EE_{s,a,s' \sim \Dcal_{t}} \log M(s'\mid s,a) ,
    \end{align*}
    and let $\hat \pi$ such that
    \begin{align*}
        \EE_{s \sim d_{e}}\left[V^{\pi_t}_{\hat M_t}(s) - \EE_{a \sim \pi^\ast(s)}[Q^{\pi_e}_{\hat M}(s,a)]\right] \leq \epsilon_{po},
    \end{align*}
    where $\pi_e$ is the explore policy, repeating the argument in the proof of \pref{thm:policy} completes the proof.
\end{proof}

\subsection{Proof for \pref{sec:moment}}
Once again, we prove the expert distribution version for a cleaner result.
\begin{theorem}[\pref{thm:mmregret} restate]
  Let $\{\hat \pi_t\}^T_{t=1}$ be the sequence of returned policies of
  \ouralgmm, we have:
    \begin{align*}
      \frac{1}{T}\sum_{t=1}^T J_{\Mstar}^\omega(\pihat_t) -
                        J_{\Mstar}^\omega(\pistar)
                        \leq  \tilde O\left( \epsilon_{po} +  \frac{1
                        }{1-\gamma} \left(\epsilon^{mm}_{model} +
                        \frac{1}{\sqrt{T}}\right)\right),
    \end{align*}
    where
    $\epsilon^{mm}_{model} = \min_{M \in \Mcal}
    \frac{1}{T}\sum_{t=1}^T \ell_t(M)$ is the agnostic model error.
   %\todo[inline]{@Yuda: What is $V$ in the definition of
    %  $\epsilon^{mm}_{model}$? We need to remove dependence on $\max_t$}
\end{theorem}

\begin{proof}
For simplicity, we only prove the expert distribution version. Again we start with \pref{lem:simulation-pdl}:
\begin{align*}
        &J_{\Mstar}^\omega(\pihat) - J_{\Mstar}^\omega(\pistar) = \mathbb{E}_{s \sim \omega} [V_{M^\star}^{\hat{\pi}}(s) - V_{M^\star}^{\pi^\star}(s)] \\
        =&\frac{\gamma}{1 - \gamma}\mathbb{E}_{(s, a) \sim D_{\omega, \pihat}}[\expec_{s' \sim \Mstar(s, a)}[V_{\Mhat}^{\pihat}(s')] - \expec_{s'' \sim \Mhat(s, a)}[V_{\Mhat}^{\pihat}(s'')]] \\
        +&\frac{\gamma}{1 - \gamma}\expec_{(s, a) \sim D_{\omega, \pistar}}[\expec_{s'' \sim \Mhat(s, a)}[V^{\pihat}_{\Mhat}(s'')] - \expec_{s' \sim \Mstar(s, a)}[V^{\pihat}_{\Mhat}(s')]] \\
        +&\frac{1}{1 - \gamma}\expec_{(s,a) \sim D_{\omega, \pistar}}[ V^{\pihat}_{\Mhat}(s) - Q^{\pihat}_{\Mhat}(s, a)],
    \end{align*}
Also similar to the previous proofs,
\begin{align}\label{eq:appmmpolicy}
        & \sum_{t=1}^T J_{\Mstar}^\omega(\pihat^t) - J_{\Mstar}^\omega(\pistar) \nonumber \\
        \leq& \sum_{t=1}^{T}\frac{\gamma}{1 - \gamma}\mathbb{E}_{(s, a) \sim D_{\omega, \pihat_t}}[\expec_{s' \sim \Mstar(s, a)}[V_{\Mhat_t}^{\pihat_t}(s')] - \expec_{s'' \sim \Mhat_t(s, a)}[V_{\Mhat_t}^{\pihat_t}(s'')]]  \nonumber\\
        +&\sum_{t=1}^{T}\frac{\gamma}{1 - \gamma}\expec_{(s, a) \sim D_{\omega, \pistar}}[\expec_{s'' \sim \Mhat_t(s, a)}[V^{\pihat_t}_{\Mhat_t}(s'')] - \expec_{s' \sim \Mstar(s, a)}[V^{\pihat_t}_{\Mhat_t}(s')]]  \nonumber\\
        +&\sum_{t=1}^{T}\frac{1}{1 - \gamma}\expec_{(s,a) \sim D_{\omega, \pistar}}[ V^{\pihat_t}_{\Mhat_t}(s) - Q^{\pihat_t}_{\Mhat_t}(s, a)]  \nonumber\\
        \leq& \sum_{t=1}^{T}\frac{\gamma}{1 - \gamma}\mathbb{E}_{(s, a) \sim D_{\omega, \pihat_t}}[\expec_{s' \sim \Mstar(s, a)}[V_{\Mhat_t}^{\pihat_t}(s')] - \expec_{s'' \sim \Mhat_t(s, a)}[V_{\Mhat_t}^{\pihat_t}(s'')]]  \nonumber\\
        +&\sum_{t=1}^{T}\frac{\gamma}{1 - \gamma}\expec_{(s, a) \sim D_{\omega, \pistar}}[\expec_{s'' \sim \Mhat_t(s, a)}[V^{\pihat_t}_{\Mhat_t}(s'')] - \expec_{s' \sim \Mstar(s, a)}[V^{\pihat_t}_{\Mhat_t}(s')]]  \nonumber\\
        +&T \epsilon_{po}
    \end{align}
Now we see how the model learning part actually helps us bound the first two terms. Recall our loss in \pref{eq:absolute}
%\aarti{should $\Mhat^t$ be $\Mhat_t$ below?}
\begin{align*}
      \ell_t(\Mhat_t) = \EE_{s,a \sim \Dcal_{t}}  \left|\expec_{s' \sim \Mstar(s, a)}[V_{\Mhat_t}^{\pihat_t}(s')] - \expec_{s'' \sim M(s, a)}[V_{\Mhat_t}^{\pihat_t}(s'')] \right|,
\end{align*}
Upper bounding terms in~\eqref{eq:appmmpolicy} using the loss
$\ell_t(\Mhat_t)$ we get
\begin{align*}
  & \sum_{t=1}^T J_{\Mstar}^\omega(\pihat^t) -
    J_{\Mstar}^\omega(\pistar) \nonumber \\
  &\leq\frac{2\gamma}{1-\gamma}\sum_{t=1}^T \ell_t(\Mhat_t) + T\epsilon_{po}
\end{align*}
Using FTRL for the sequence of losses $\ell_t(\Mhat_t)$ gives  us:
\begin{align*}
  \sum_{t=1}^T \ell_t(\Mhat_t) &\leq \min_{M \in \Mcal} \sum_{t=1}^T
  \ell_t(M) + O(\sqrt{T}) \\
  &\leq T\epsilon_{model}^{mm} + O(\sqrt{T})
\end{align*}
Substituting this above we get
\begin{align*}
& \sum_{t=1}^T J_{\Mstar}^\omega(\pihat^t) -
  J_{\Mstar}^\omega(\pistar) \nonumber \\
  &\leq \frac{2\gamma}{1-\gamma} (T\epsilon_{model}^{mm} +
    O(\sqrt{T})) + T\epsilon_{po}
\end{align*}
\iffalse
and the FTL result guarantees us that:
\begin{align*}
    \sum_{t=1}^T \ell_t(\hat M_t) &\leq \min_{M \in \Mcal}\sum_{t=1}^T \ell_t(M) + O(\log{T}) \nonumber \\
    &\leq T\epsilon_{model}^{mm} + O(\log{T}).
\end{align*}
Note that this is valid because at the beginning of iteration $t$, we already know $\{\hat M_{\tau}\}_{\tau=1}^t$ (we compute for $\hat M^{t+1}$), and the sequence of policies is assumed to be arbitrary.
Then we have
\begin{align*}
    \sum_{t=1}^{T}\expec_{(s, a) \sim \Dcal_t}[\expec_{s'' \sim \Mhat_t(s, a)}[V^{\pihat_t}_{\Mhat_t}(s'')] - \expec_{s' \sim \Mstar(s, a)}[V^{\pihat_t}_{\Mhat_t}(s')]] &\leq \sum_{t=1}^{T}\expec_{(s, a) \sim \Dcal_t}\left|\expec_{s'' \sim \Mhat_t(s, a)}V^{\pihat_t}_{\Mhat_t}(s'') - \expec_{s' \sim \Mstar(s, a)}V^{\pihat_t}_{\Mhat_t}(s')\right| \\
    &\leq \sum_{t=1}^{T}\expec_{(s, a) \sim \Dcal_t}\sqrt{\left(\expec_{s'' \sim \Mhat_t(s, a)}V^{\pihat_t}_{\Mhat_t}(s'') - \expec_{s' \sim \Mstar(s, a)}V^{\pihat_t}_{\Mhat_t}(s')\right)^2} \\
    &\leq T\sqrt{\frac{1}{T}\sum_{t=1}^{T}\expec_{(s, a) \sim \Dcal_t}\left(\expec_{s'' \sim \Mhat_t(s, a)}V^{\pihat_t}_{\Mhat_t}(s'') - \expec_{s' \sim \Mstar(s, a)}V^{\pihat_t}_{\Mhat_t}(s')\right)^2} \\
    &\leq T\sqrt{\epsilon_{model}^{mm}} + \tilde O(\sqrt{T}).
\end{align*}
Then plug this back to \pref{eq:appmmpolicy} gives us:
\begin{align*}
    J_{\Mstar}^\omega(\pihat, T) - J_{\Mstar}^\omega(\pistar, T) \leq T \epsilon_{po} + T\sqrt{ \epsilon_{model}^{mm}} + \tilde O(\sqrt{T}).
\end{align*}
\fi
And once again multiply both sides by $\frac{1}{T}$ and using $\gamma
\leq 1$ completes the proof.
\end{proof}

\section{Additional Analysis}\label{app:analysis}
In this section, we present a few deferred results from the main text.

\subsection{A warm-up argument using Hedge}\label{sec:hedge}
To see the intuition of the no-regret result from the data
aggregation, let us now consider a simplified version: suppose that we
have a model class $\Mcal$ with finitely many models, and denote the
number of models as $N$. For each model $\hat M \in \Mcal$, denote a
policy with a non-positive disadvantage over $\pi^\ast$ in the model
$\hat M$ as $\pi^{\hat M}$. Now consider our proposed algorithm with
hedge as the no-regret algorithm: for each iteration $t$, we first
sample a model $\hat M_t$ according to the current weight and then
roll out with $\pi_t = \pi^{\hat M_t}$. Then for each model $\hat M$,
we compute the loss
$\ell_t(\hat M) = I\{\EE_{d^{\pi_t}}\|\hat M(s,a) - M^\ast(s,a)\|_1 +
\EE_{d^{\pi^\ast}}\|\hat M(s,a) - M^\ast(s,a)\|_1 > 0\}$. One can
think of such loss as whether a model makes any mistakes on the
current trajectory distribution. Let us assume that the model class
$\Mcal$ is realizable. Then we have:
\begin{align*}
    R(T) &\leq \sum_{t=1}^T \EE_{\hat M_t} \frac{\gamma V_{\max}}{1-\gamma}\EE_{d^{\pi_t}}\|\hat M(s,a) - M^\ast(s,a)\|_1 + \frac{\gamma V_{\max}}{1-\gamma} \EE_{d^{\pi^\ast}}\|\hat M(s,a) - M^\ast(s,a)\|_1 \\
    &\leq \sum_{t=1}^T \EE_{\hat M_t}\ell_t(\hat M_t)\\
    & \leq O(\sqrt{T\log(N)}).
\end{align*}
But note that this method is computationally inefficient (because we need to compute the loss for each $\hat M \in \Mcal$ in each round).

\subsection{Comparing with signed loss}\label{app:signed}

\begin{algorithm}[t]
\caption{Moment Matching~\textbf{FitModel} with Signed Loss($\Dcal_t, \{\ell^{sn}_i\}_{i=1}^{t-1}$)}
\begin{algorithmic}[1]
\REQUIRE Data collected from learned policy so far
$\Dcal^{learned}_t$, Data collected from exploration distribution so
far $\Dcal^{exp}_t$, model class $\Mcal$, previous losses
$\{\ell^{sn}_i\}_{i=1}^{t-1}$
\STATE Define loss $\ell^{sn}_t(M)$ as follows,
\begin{align}\label{eq:signed}
  \ell^{sn}_t(M) = \expec_{(s, a, s') \sim
  \Dcal_t^{learned}}\left[V^{\pihat_t}_{\Mhat_t}(s') - \expec_{s''
  \sim M(s, a)}[V^{\pihat_t}_{\Mhat_t}(s'')]\right] + \expec_{(s, a, s') \sim
  \Dcal_t^{exp}}\left[\expec_{s'' \sim M(s,
  a)}[V^{\pihat_t}_{\Mhat_t}(s'')] - V^{\pihat_t}_{\Mhat_t}(s')\right]
\end{align}
\STATE Compute model $\Mhat_{t+1}$ using an online no-regret
algorithm that works with convex loss, such as FTRL,
\begin{align*}
    \hat M_{t+1} \leftarrow \argmin_{M \in \Mcal} \sum_{\tau=1}^t  \ell^{sn}_\tau(M) + \Rcal(M).
\end{align*}
\STATE \textbf{Return} $\hat M_{t+1}$
\end{algorithmic}\label{alg:signed}
\end{algorithm}

In this section, we present an alternative algorithm of \ouralgmm{} that
uses a signed version of the loss instead of the squared loss \pref{eq:absolute}.
We present this alternative algorithm in \pref{alg:signed}, where we run \pref{alg:meta} with \pref{alg:policy} and \pref{alg:signed}. For simplicity, below provide the regret result in the expert distribution:

\begin{theorem}
  Let $\{\hat \pi_t\}^T_{t=1}$ be the sequence of returned policies of
  \pref{alg:signed}, we have:
    \begin{align*}
      \frac{1}{T}\sum_{t=1}^T J_{\Mstar}^\omega(\pihat_t) -
                        J_{\Mstar}^\omega(\pistar)
                        \leq  O\left( \epsilon_{po} +  \frac{1
                        }{1-\gamma} \left(\epsilon^{sn}_{model} +
                        \frac{1}{\sqrt{T}}\right)\right),
    \end{align*}
    where
    $\epsilon^{sn}_{model} = \min_{M \in \Mcal}
    \frac{1}{T}\sum_{t=1}^T \ell^{sn}_t(M)$ is the agnostic model error.
   %\todo[inline]{@Yuda: What is $V$ in the definition of
    %  $\epsilon^{mm}_{model}$? We need to remove dependence on $\max_t$}
\end{theorem}
\begin{proof}
The proof is essentially the same as the proof of \pref{thm:mmregret}. We have
\begin{align*}
    &\sum_{t=1}^T J_{\Mstar}^\omega(\pihat^t) - J_{\Mstar}^\omega(\pistar) \\
    \leq& \sum_{t=1}^{T}\frac{\gamma}{1 - \gamma}\mathbb{E}_{(s, a) \sim D_{\omega, \pihat_t}}[\expec_{s' \sim \Mstar(s, a)}[V_{\Mhat_t}^{\pihat_t}(s')] - \expec_{s'' \sim \Mhat_t(s, a)}[V_{\Mhat_t}^{\pihat_t}(s'')]]  \nonumber\\
        &+\sum_{t=1}^{T}\frac{\gamma}{1 - \gamma}\expec_{(s, a) \sim D_{\omega, \pistar}}[\expec_{s'' \sim \Mhat_t(s, a)}[V^{\pihat_t}_{\Mhat_t}(s'')] - \expec_{s' \sim \Mstar(s, a)}[V^{\pihat_t}_{\Mhat_t}(s')]]  \nonumber\\
        &+T \epsilon_{po}\\
    =& \sum_{t=1}^T \ell^{sn}_t(\hat M_t) + T \epsilon_{po}.
\end{align*}
and using FTRL gives us:
\begin{align*}
    \sum_{t=1}^T \ell_t^{sn}(\hat M_t) &\leq \min_{M \in \Mcal}\sum_{t=1}^T \ell^{sn}_t(M) + O(\sqrt{T}) \nonumber \\
    &\leq T\epsilon_{model}^{mm} + O(\sqrt{T})).
\end{align*}
And taking $\frac{1}{T}$ on both sides completes the proof.
\end{proof}
We remark that here we see that both loss function gives us a $\tilde O(\frac{1}{\sqrt{T}})$ regret rate.

\subsection{Finite Sample Analysis of \ouralgmm}\label{sec:finite}

In this section, we perform a finite sample analysis of \pref{alg:model_mm} using the online-to-batch technique \citep{cesa2004generalization}. First, let's introduce a new function class. This function class is constructed with the model class $\Mcal$, and it takes state, action, and value function triplets as inputs. Denote $\Xcal = \Scal \times \Acal \times \Vcal$, where $\Vcal$ is the function class,
\begin{align*}
    \Hcal = \biggl\{h: \Xcal \to \RR \mid \exists M \in \Mcal
    ~~~\text{s.t.}~~~ \forall (s,a,v) \in \Xcal, h(s,a,v) = \int M(s'\mid s,a)v(s') \;d(s')\biggr\}.
\end{align*}
Denote random variable $x_t = (s_t,a_t,v_t)$, we note the generation of the random variable $y_t$ where
\begin{align*}
    y_t = v_t(s'_t), ~~~~ s'_t \sim M^\ast(s_t,a_t).
\end{align*}
Denote $\Fcal_t = \{(X_1, Y_1), \dots, (X_{t-1}, Y_{t-1})\}$, and at each round, we use the loss function
\begin{align*}
    \hat \ell_t(h) = |h(s_t,a_t,v_t) - y_t| + |h(\tilde s_t,\tilde a_t,v_t) - \tilde y_t)|,
\end{align*}
where $(s,a) \sim \Dcal_t, s' \sim M^\ast(s,a)$ and $(\tilde s,\tilde a) \sim \Dcal_{\pi^\ast}, \tilde s' \sim M^\ast(\tilde s,\tilde a)$. Further, define
\begin{align*}
    Z_t = (\hat \ell(\hat h_t) - \ell(\hat h_t)) -  (\hat \ell(h^\ast) - \ell(h^\ast)).
\end{align*}
Note that $Z_t$ is a martingale difference sequence adapted to the filtration $\Fcal_t$, such that
\begin{align*}
    \EE[Z_t \mid \Fcal_t] = 0.
\end{align*}
Meanwhile, we also have that $|Z_t| \leq \frac{4}{1-\gamma}$. Then by \pref{lem:azuma}, we have with probability $1-\delta$,
\begin{align*}
    \sum_{t=1}^T {Z_t} \leq \sqrt{\frac{32T\log(1/\delta)}{(1-\gamma)^2}},
\end{align*}
then taking a union bound on $\Hcal$, we have for any $h$,
\begin{align*}
    \sum_{t=1}^T {Z_t} \leq \sqrt{\frac{32T\log(|\Hcal|/\delta)}{(1-\gamma)^2}}.
\end{align*}
Then by \pref{lem:conversion} and realizability, we have
\begin{align*}
\sum_{t=1}^T \ell_t(\hat h_t) &\leq R(T) +  \sum_{t=1}^T \ell_t(h^\ast) + Z_t \\
&\lesssim \frac{1}{1-\gamma}\sqrt{ T\log(|\Hcal|/\delta) },
\end{align*}
since $R(T) = O(\sqrt{T})$ as well. For simplicity, let's further assume that $\epsilon_{po} \leq 0$, then we have the following regret bound:
\begin{align*}
    \frac{1}{T}\sum_{t=1}^T J_{\Mstar}^\omega(\pihat_t) - J_{\Mstar}^\omega(\pistar) &\leq \frac{1}{T}\sum_{t=1}^T \ell_t(\hat h_t) \\
    &\leq \tilde O\left( \frac{T^{1/2}\log(|\Hcal|/\delta)^{1/2}}{(1-\gamma)} \right),
\end{align*}
Converting to sample complexity we have, by taking
\begin{align*}
    T = \tilde O\left(\frac{\log(|\Hcal|/\delta)}{(1-\gamma)^2\epsilon^2}\right),
\end{align*}
we have with probability $1-\delta$,
\begin{align*}
   \frac{1}{T} \sum_{t=1}^T J_{\Mstar}^\omega(\pihat_t) - J_{\Mstar}^\omega(\pistar) \leq \epsilon.
\end{align*}
Here we add a few remarks. First is that this result does not directly compare to the traditional MBRL sample complexity because a tight bound on the size of $\Hcal$ is instance-dependent. Second, this result does not contradict to the difficulties mentioned in \pref{sec:moment} because the issue of sampling from the learned model is implicitly addressed by the construction of $\Hcal$. %Finally, we remark that this result may not be optimal, one may be able to obtain a faster rate by leveraging the fact that the loss function here is simply performing least square regression \citep{beygelzimer2011contextual}.
\subsection{Auxiliary Lemmas}

\begin{lemma}[Hoeffding-Azuma Inequality]\label{lem:azuma}
Suppose $X_1, \dots, X_T$ is a martingale difference sequence where $|X_t| \leq R$. Then for all $\epsilon > 0$ and all positive integer T, we have
\begin{align*}
    P \left( \sum_{t=1}^T X_i \geq \epsilon \right) \leq \exp \left( \frac{-\epsilon^2}{2TR^2} \right).
\end{align*}
\end{lemma}

\begin{lemma}[Online-to-batch Conversion]\label{lem:conversion}
   Consider a sequential function estimation problem with function class $\Hcal$. Let $\Xcal$ be the input space and $\Ycal$ be the target space. Assume each the inputs and targets $(x_t,y_t)$ are generated i.i.d., where $x_t \sim \rho(x_1,y_1,\dots,x_{t-1},y_{t-1})$, $y \sim p^\ast(\cdot \mid x_t)$. Let $\hat h_t$ be the return of an online learning algorithm $A$, taking inputs $\{x_1,y_1,\hat \ell_1,\dots,x_{t-1},y_{t-1},\hat \ell_{t-1}\}$, where $\hat \ell_t$ is the empirical version of loss function $\ell_t$ at round $t$. Define
   \begin{align*}
    Z_t = (\hat \ell(\hat h_t) - \ell(\hat h_t)) -  (\hat \ell(h^\ast) - \ell(h^\ast)),
\end{align*}
where $h^\ast = \min_{h \in \Hcal} \sum_{t=1}^T \ell_t(h)$, we have
\begin{align*}
    \sum_{t=1}^T \ell_t(\hat h_t) \leq R(T) + \sum_{t=1}^T \ell_t(h^\ast) + Z_t,
\end{align*}
where $R(T)$ is the regret of running $A$ at round T.
\end{lemma}

%%% Local Variables:
%%% mode: latex
%%% TeX-master:  "../example_paper"
%%% End:

\section{Experiment details} \label{app:exp}

\subsection{Environment Details}
In this section, we provide details on the environments we used in
\pref{sec:experiment}, epecially the non-standard benchmarks such as
the helicopter and WideTree.

\subsubsection{Helicopter}\label{app:helicopter}
The helicopter domain is first proposed in
\citet{abbeel2005exploration} and is also used in
\citet{ross2012agnostic}. In this paper we focus on the \textit{hover}
task. The environment has a 20-dimensional state space and
4-dimensional action space.

% The dynamics of the system follows:
% \begin{align*}
%   \Delta_{x_{t+1}} = A\Delta_{x_t} + B\Delta_{u_t},
% \end{align*}
% where $\Delta_t$ is the difference between the state at timestep $t$
% and hover state, $\Delta_{u_t}$ is the delta control at timestep $t$,
% $A$ is identity and $B$ adds the delta control to the actual control
% in the state delta.
The dynamics of the system are nonlinear and are parameterized using
mass and inertial quantities as a $20$-dimensional vector. The model
class is $\mathbb{R}^{20}$ and corresponds to the parameter vector
used to define the dynamics.
The cost function is
\begin{align*}
  c(x,u) = \sum_{h=1}^H x^\top_h Q x_h + u^\top_h R u_h + x^\top_H Q_H x_H,
\end{align*}
i.e., we penalize any deviation from the origin, and any control
effort expended.

To understand why a single backward pass on the desired trajectory
would be equivalent to~\pref{alg:policy}, let us revisit the objective
in~\pref{alg:policy}:
$$ \expec_{s \sim \nu}\left[V^{\pihat_t}_{\Mhat_t}(s) - \min_{a \in
    \Acal}[Q^{\pihat_t}_{\Mhat_t}(s,a)]\right] \leq \epsilon_{po} $$
In the above objective, $\nu$ is the exploration distribution which in
this case, is simply the desired trajectory that keeps the trajectory
at hover ($\{s_{hover}, u_{hover}, \ldots, s_{hover}\}$). $\Mhat_t$ is
the model we are optimizing in, and the above objective states that we
need to find a policy $\pihat_t$ that is as good as the optimal policy
\textit{only} on the desired trajectory. Thus, to computet this we
linearize the nonlinear dynamics of $\Mhat_t$ around the desired
trajectory (forward pass) and then compute the optimal LQR controller
for the linearized dynamics (backward pass.) This gives us a policy
$\pihat_t$ that is as good as optimal \textit{only} along the desired
trajectory. Note that this requires a single backward pass while
achieving~\eqref{eq:oc} requires multiple iterations of iLQR involving
multiple bcakward passes. This highlights the computational advantage
of~\pref{alg:policy} over traditional optimal planning methods.

\subsubsection{WideTree}\label{app:widetree}
This MDP is a variant of the one showed in~\pref{fig:toy}. It is
described in~\pref{fig:widetree}.

\begin{figure}[h]
  \centering
  \includegraphics[width=0.5\linewidth]{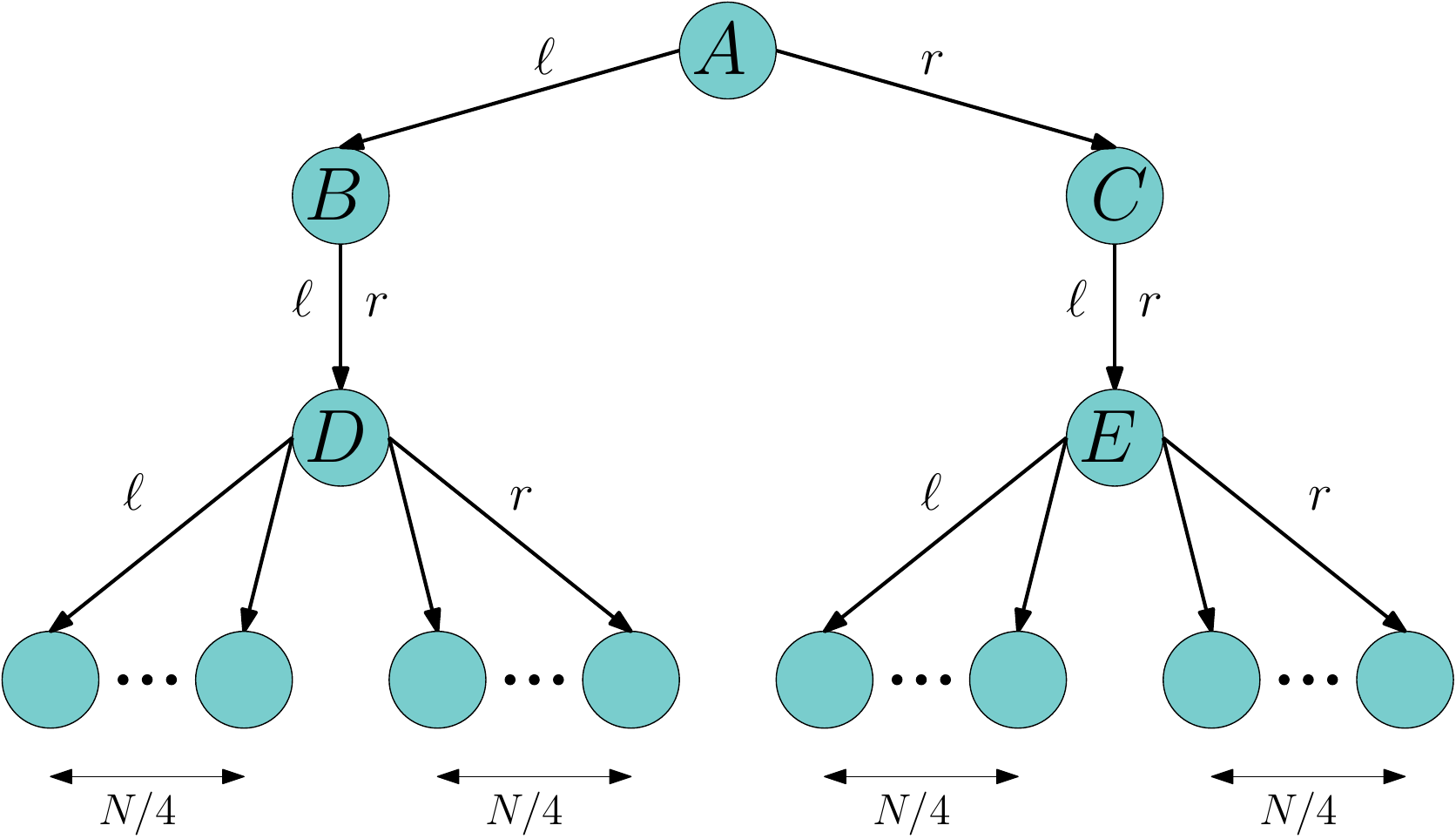}
  \caption{The Widetree domain used in experiments. We have an MDP
    with $N+5$ states where $N$ of them are terminal states (or
    leaves.) Each state has two actions $\ell$ and $r$. At states $B$
    and $C$, both actions lead to the same state $D$ and $E$
    respectively. At states $D$ and $E$, both actions stochastically
    transition the state to one of $N/4$ terminal states with uniform
    probability. The true dynamics $\Mstar$ is shown in the
    figure. The cost $c(s,a) = \epsilon<<1$ at any $s \neq B$ and
    $c(B, a) = 1$, for any action $a$. Thus, the action taken at $A$
    is critical. Model class $\Mcal$ contains only two models:
    $M^{\mathsf{good}}$ which captures dynamics at $A$ correctly but
    makes mistakes everywhere else, while $M^{\mathsf{bad}}$ makes
    mistakes only at $A$ but captures true dynamics everywhere else.}
  \label{fig:widetree}
\end{figure}

We implement~\pref{alg:model_mm} using Hedge~\cite{freund1997hedge} by
maintaining a discrete distribution $(p, 1-p)$ over the two models
$M^{\mathsf{good}}$ and $M^{\mathsf{bad}}$. We use $\epsilon=0.9$ for
the hedge update.

\subsubsection{Linear Dynamical System}
\label{app:lds}
In this experiment, the task is to control a linear dynamical system
where the true dynamics are time-varying but the model class only
contains time-invariant linear dynamical models. The system has a 5-D
state, a 1-D control input, and we are tasked with controlling it for
a horizon of $100$ timesteps. The true dynamics
evolves according to $x_{t+1} = A_tx_t + B_tu_t$ where,
\begin{align*}
A_t =
  \begin{bmatrix}
    0.5 & 0 \\ 0 & 0.5
  \end{bmatrix}
\end{align*}
when $t$ is even and,
\begin{align*}
  A_t =
  \begin{bmatrix}
    1.5 & 0 \\ 0 & 1.5
  \end{bmatrix}
\end{align*}
when $t$ is odd. The model class $\Mcal$ consists of linear dynamical
models of the form $\{x_{t+1} = Ax_t + Bu_t\}$ and thus cannot model
the true dynamics exactly making it an agnostic model class.

The cost function is quadratic as follows,
\begin{align*}
  c(\mathbf{x}, \mathbf{u}) = \sum_{t=0}^{99} u_t^Tu_t + x_{100}^Tx_{100}
\end{align*}
where $\mathbf{x}, \mathbf{u}$ represent the entire state and control
trajectory. Note that the cost function only penalizes the control
input at every timestep and penalizes the state only at the final
timestep.

Given a model $(A, B) \in \Mcal$, we can compute the optimal policy
and its value function in closed form by using the finite horizon
discrete ricatti solution~\cite{bertsekas2005dynamic}. The value
function is represented using matrices $P_t \in \mathbb{R}^{5 \times 5}$
where $V_t(x) = x^TP_tx$ denotes the cost to go from time $t$ until
the end of horizon. Thus, we can construct the loss for any model $M =
(A, B)$
in~\pref{alg:model_mm} for \ouralgmm{} as,
\begin{align*}
  \ell_t(A, B) = \expec_{(x_h, u_h, x_{h+1}) \sim \Dcal_t}
  (x_{h+1}^TP_{h+1}x_{h+1} - (Ax_h + Bu_h)^TP_{h+1}(Ax_h + Bu_h))^2
\end{align*}
whereas the MLE loss for \sysid{} is simply,
\begin{align*}
  \ell_t(A, B) = \expec_{(x_h, u_h, x_{h+1}) \sim \Dcal_t} \|x_{h+1} -
  (Ax_h + Bu_h)\|_2^2
\end{align*}

\subsubsection{Maze}\label{app:maze}
\begin{figure}
  \centering
  \includegraphics[width=0.3\linewidth]{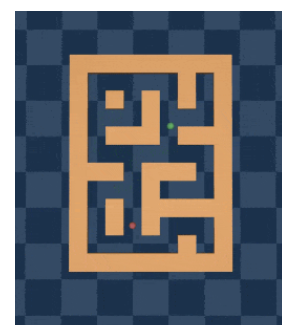}
  \caption{The PointMaze (large) environment. Picture taken
    from \url{https://sites.google.com/view/d4rl/home}.}
  \label{fig:pointmaze}
\end{figure}
For the maze environment, we adopt the PointMaze (large) task from
D4RL \citep{fu2020d4rl}. We present a visualization of the task in
\pref{fig:pointmaze}. While the original offline dataset contains
4000000 samples, we only take 10000 and 50000 samples in our
experiment.

\subsection{Additional Mujoco Experiment} \label{app:failure} In this
section, we show an additional mujoco experiment. In this case, our
algorithm is outperformed by \mbposysid{} initially and reaches the
same performance given enough data. We hypothesize that the main cause
for this is that the explore distribution does not
have high quality in this case, which suggests that
\pref{alg:policy} is more sensitive to the quality of the exploration
distribution than \mbposysid{}, as described in~\pref{sec:alg_mle}.

\begin{figure}
  \centering
  \includegraphics[width=0.5\linewidth]{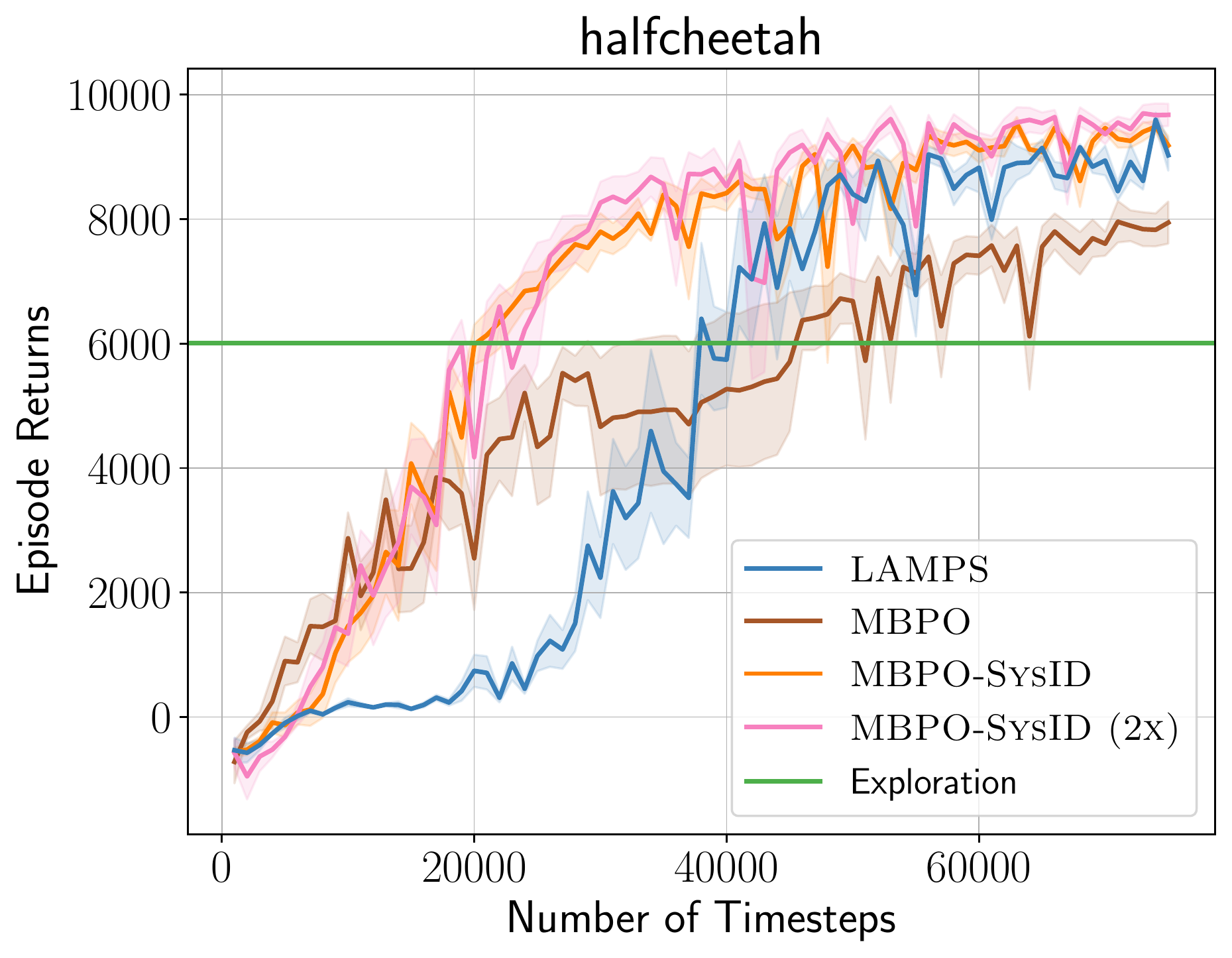}
  \caption{Result on the Halfcheetah benchmark. The results are
    average over 5 random seeds and the shaded area denotes the
    standard error. We use 10000 exploration samples. Note that in
    this case, \ouralg{} takes more sample to reach the asymptotic
    performance of \mbposysid{}.}
  \label{fig:cheetah}
\end{figure}

\subsection{Implementation Details and Hyperparameters}\label{app:mujoco}
In this section, we provide the implementation details and
hyperparameter we used in our
experiments in \pref{sec:gym} and \pref{sec:maze}. As mentioned in the
main text, our
implementation for \mbpo{} is based on \citet{pineda2021mbrl}, so does
\mbposysid{} and \ouralg{}. For model training, all baselines use the same ensemble
of models as in the original \mbpo, while \mbposysid{} and \ouralg{} use both
the data aggregation and exploration data to train the model.
For the policy training,
\mbposysid{} uses the same branch update as in \mbpo{} with Soft Actor-Critic (SAC), and
\ouralg{} uses a different objective by replacing $Q$-value with the disadvantage term
during the actor update step (note that this implies that the actor update is also
only based
on the exploration distribution.) We present the detailed algorithm in \pref{alg:deep}.

{\begin{algorithm}[t]
\caption{Deep \ouralg}
\begin{algorithmic}[1]
\REQUIRE Number of iterations $T$, model ensemble $M_\theta$, policy
$\pi_\phi$, value function $q_\psi$, exploration dataset $\nu$, policy rollout step $H$, model rollout $E$, model rollout horizon $k$.
\STATE Initialize data aggregation $\Dcal = \emptyset$.
\FOR{$t = 1, \dots, T$}
\STATE Train model ensemble $M_\theta$ with MLE with $\rho = \frac{1}{2}\nu + \frac{1}{2}\Dcal$: with $\frac{1}{2}$ probability sampling from $\nu$ and $\frac{1}{2}$ probability sampling from $\Dcal$.
\FOR{$h = 1, \dots, H$}
%\textcolor{blue}{\# Collect online dataset}  \\
\STATE Collect data in $\Mstar$ by rolling out $\pi_\phi$
\STATE Initialize model buffer $\Dcal_{model} = \emptyset$.
\FOR{$e = 1, \dots, E$}
\STATE Sample state $s$ uniformly from $\rho$, rollout $k$ step with $\pi_\theta$ and add trajectory to $\Dcal_{model}$
\STATE Update soft $Q$-value function $q_\psi$ with $\Dcal_{model}$.
\STATE Update policy $\pi_\phi$ with
\begin{align*}
    \Jcal(\pi_\phi,\nu,q_\psi) = \EE_{s,a \sim \nu}\left[ \EE_{\tilde a \sim \pi_\phi(s)} \log(\pi_\phi(\tilde a \mid s)) - (q_\psi(s,\tilde a) - q_\psi(s,a))\right]
\end{align*}
\ENDFOR
\ENDFOR
\ENDFOR
\STATE \textbf{Return} Sequence of policies $\{\pihat_t\}_{t=1}^{T+1}$
\end{algorithmic}\label{alg:deep}
\end{algorithm}}

We use the default
hyperparameter for most case, but we present them for
completeness. Note that the hyperparameters are the same for all
baselines, but \mbposysidcompute{} uses double the number indicated with
the hyperparameter ends with $(\ast)$.

\begin{table}[h]
  \caption{Hyperparameters for HalfCheetah}
  \centering
  \begin{tabular}{cc}
    \toprule
    &\; Value  \\
    \hline
    Exploration sample size                                  &\; 10000  \\
    Ensemble size                                            &\; 7      \\
    Ensemble elite number                                    &\; 5      \\
    Model learning rate                                      &\; 0.001   \\
    Model batch size                                         &\; 256     \\
    Rollout step in learned model ($\ast$)                   &\; 400     \\
    Rollout  length                                          &\; $1 \to 1$     \\
    Number policy updates ($\ast$)                           & 20       \\
    Policy type                                              & Stochastic Gaussian Policy \\
    \toprule
  \end{tabular}\label{table:halfcheetah}\end{table}

\begin{table}[h]
  \caption{Hyperparameters for Ant}
  \centering
  \begin{tabular}{cc}
    \toprule
    &\; Value  \\
    \hline
    Exploration sample size                                  &\; 10000  \\
    Ensemble size                                            &\; 7      \\
    Ensemble elite number                                    &\; 5      \\
    Model learning rate                                      &\; 0.0003   \\
    Model batch size                                         &\; 256     \\
    Rollout step in learned model ($\ast$)                   &\; 400     \\
    Rollout  length                                          &\; $1 \to 25$     \\
    Number policy updates ($\ast$)                           & 20       \\
    Policy type                                              & Stochastic Gaussian Policy \\
    \toprule
  \end{tabular}\label{table:ant}\end{table}

\begin{table}[h]
  \caption{Hyperparameters for Hopper}
  \centering
  \begin{tabular}{cc}
    \toprule
    &\; Value  \\
    \hline
    Exploration sample size                                  &\; 10000  \\
    Ensemble size                                            &\; 7      \\
    Ensemble elite number                                    &\; 5      \\
    Model learning rate                                      &\; 0.001   \\
    Model batch size                                         &\; 256     \\
    Rollout step in learned model ($\ast$)                   &\; 400     \\
    Rollout  length                                          &\; $1 \to 15$     \\
    Number policy updates ($\ast$)                           & 40       \\
    Policy type                                              & Stochastic Gaussian Policy \\
    \toprule
  \end{tabular}\label{table:hopper}\end{table}

\begin{table}[h]
  \caption{Hyperparameters for Humanoid}
  \centering
  \begin{tabular}{cc}
    \toprule
    &\; Value  \\
    \hline
    Exploration sample size                                  &\; 10000  \\
    Ensemble size                                            &\; 7      \\
    Ensemble elite number                                    &\; 5      \\
    Model learning rate                                      &\; 0.0003   \\
    Model batch size                                         &\; 256     \\
    Rollout step in learned model ($\ast$)                   &\; 400     \\
    Rollout  length                                          &\; $1 \to 25$     \\
    Number policy updates ($\ast$)                           & 20       \\
    Policy type                                              & Stochastic Gaussian Policy \\
    \toprule
  \end{tabular}\label{table:humanoid}\end{table}

\begin{table}[h]
  \caption{Hyperparameters for Walker}
  \centering
  \begin{tabular}{cc}
    \toprule
    &\; Value  \\
    \hline
    Exploration sample size                                  &\; 10000  \\
    Ensemble size                                            &\; 7      \\
    Ensemble elite number                                    &\; 5      \\
    Model learning rate                                      &\; 0.001   \\
    Model batch size                                         &\; 256     \\
    Rollout step in learned model ($\ast$)                   &\; 400     \\
    Rollout  length                                          &\; $1 \to 1$     \\
    Number policy updates ($\ast$)                           & 20       \\
    Policy type                                              & Stochastic Gaussian Policy \\
    \toprule
  \end{tabular}\label{table:walker}\end{table}

\begin{table}[h]
  \caption{Hyperparameters for PointMaze}
  \centering
  \begin{tabular}{cc}
    \toprule
    &\; Value  \\
    \hline
    Exploration sample size                                  &\; 10000/50000  \\
    Ensemble size                                            &\; 7      \\
    Ensemble elite number                                    &\; 5      \\
    Model learning rate                                      &\; 0.001   \\
    Model batch size                                         &\; 256     \\
    Rollout step in learned model ($\ast$)                   &\; 400     \\
    Rollout  length                                          &\; $1 \to 1$     \\
    Number policy updates ($\ast$)                           & 20       \\
    Policy type                                              & Deterministic Policy \\
    \toprule
  \end{tabular}\label{table:maze}\end{table}
%%% Local Variables:
%%% mode: latex
%%% TeX-master: "../example_paper"
%%% End:

%\bibliographystyle{plain}

%\input{sec_appendix}

\end{document}